\documentclass{article}

\usepackage{PRIMEarxiv}

\usepackage[utf8]{inputenc} 
\usepackage[T1]{fontenc}    
\usepackage{hyperref}       
\usepackage{url}            
\usepackage{booktabs}       
\usepackage{amsfonts}       
\usepackage{nicefrac}       
\usepackage{microtype}      
\usepackage{lipsum}
\usepackage{fancyhdr}       
\usepackage{graphicx}       
\graphicspath{{media/}}     

\usepackage{amsmath, amsthm}
\RequirePackage{algorithm2e} 
\RequirePackage{enumitem} 
\usepackage{caption} 
\usepackage{subcaption} 

\SetKwFor{RepTimes}{repeat}{times}{end} 

\newtheorem{theorem}{Theorem}[section]
\newtheorem{lemma}[theorem]{Lemma}
\newtheorem{example}{Example}
\newtheorem{remark}{Remark}[section]
\newtheorem{assumption}{Assumption}
\RestyleAlgo{ruled}
\SetKwComment{Comment}{/* }{ */}
\SetKwInOut{Require}{Require}
\SetKwInOut{Input}{Input}
\SetKwInOut{Output}{Output}
\SetKwInOut{Update}{Update}
\SetKwInOut{Return}{Return}

\title{Off-policy reinforcement learning with high dimensional reward
}

\author{
  Dong Neuck Lee \\
  Department of Biostatistics \\
  University of North Carolina at Chapel Hill \\
  \texttt{east90@live.unc.edu} \\
   \And
  Michael R. Kosorok \\
  Department of Biostatistics \\
  University of North Carolina at Chapel Hill \\
  \texttt{kosorok@bios.unc.edu} \\
}

\begin{document}
\maketitle

\begin{abstract}
Conventional off-policy reinforcement learning (RL) focuses on maximizing the expected return of scalar rewards. Distributional RL (DRL), in contrast, studies the distribution of returns with the distributional Bellman operator in a Euclidean space, leading to highly flexible choices for utility. This paper establishes robust theoretical foundations for DRL. We prove the contraction property of the Bellman operator even when the reward space is an infinite-dimensional separable Banach space.  Furthermore, we demonstrate that the behavior of high- or infinite-dimensional returns can be effectively approximated using a lower-dimensional Euclidean space. Leveraging these theoretical insights, we propose a novel DRL algorithm that tackles problems which have been previously intractable using conventional reinforcement learning approaches.
\end{abstract}


\section{Introduction}
\label{s: intro}

Reinforcement learning methods are commonly used in various applications, such as in clinical trials, games, and economics, to find an optimal decision rule. While most of these methods have been studied to maximize the expected value of the outcome, the decision rule’s purpose is not limited to maximizing the expected value alone. We have developed a theory for a method that can be adapted to a wider range of research goals. The method based on this theory is suitable for situations where rewards are multivariate and the objective is to maximize a potentially complex utility function rather than just the expectation.

In reinforcement learning, the primary objective is to discover a policy that maximizes the expected total reward, while considering the stochastic nature of state transitions. The expected return, representing the discounted sum of the total reward function, serves as the target to maximize. Various methodologies have been explored to achieve this objective, such as Q-learning \cite{watkins1992q}, Sarsa \cite{rummery1994line}, and Deep-Q Networks \cite{rummery1994line}. These methods focus on the \emph{state-action value function (Q-function)}, which represents the expected return given the current state and action. They then optimize this Q-function based on the Bellman optimality equation.

Beyond the conventional Bellman operator, certain reinforcement learning techniques adopt the distributional Bellman operator \cite{ross2014introduction,jaquette1973markov,white1988mean}. This operator takes into account the entire distribution of potential returns, providing a more comprehensive understanding of the return landscape compared to relying solely on the average outcome. 
Algorithms based on the distributional Bellman optimality operator, such as c51 \cite{bellemare2017distributional} and quantile DQN \cite{Dabney_Rowland_Bellemare_Munos_2018}, have shown strong performance. However, like traditional reinforcement learning methods, these algorithms are also limited to finding a policy that maximizes the expectation of random returns. 

These existing reinforcement learning methods are limited in two aspects; they aim to maximize the \emph{expectation} of \emph{univariate} return. This approach may not adequately capture the complexities of real-world decision-making. Consider a trading analyst seeking a strategy that maximizes the first quartile of future assets, prioritizing risk management over average returns. A clinical study might aim to maximize drug effectiveness while constraining side effects to a specific threshold. In such cases, crafting a univariate reward that accurately reflects these multi-dimensional objectives can be challenging. Furthermore, the reward space itself can be infinite-dimensional. Consider also brain imaging data, with its continuous 3D structure and dynamic activity, inherently representing an infinite dimensional reward.
As illustrated in Figure~\ref{fig:brain}\footnote{Brain activity data are sourced from NeuroVault (https://identifiers.org/neurovault.collection:9346).}, which presents brain activity data from Chopar et al. (2021) \cite{chopra2021differentiating}, the complex and dynamic nature of brain responses highlights the need for models capable of handling high-dimensional and potentially infinite-dimensional reward spaces \cite{gorgolewski2015neurovault}.

\begin{figure}
     \centering
     \begin{subfigure}[b]{0.32\textwidth}
         \centering
         \includegraphics[width=\textwidth]{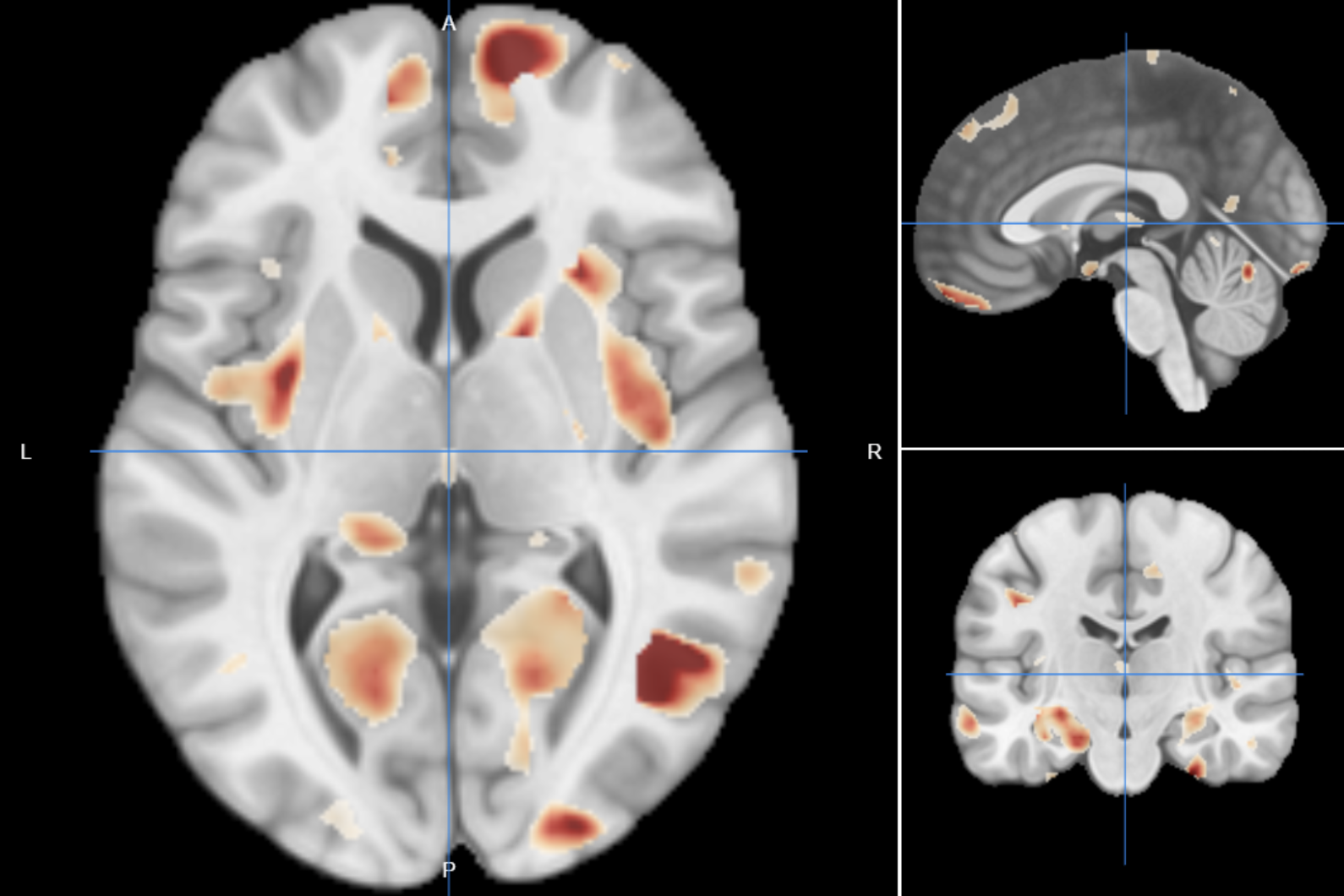}
         \caption{Baseline}
         \label{fig:brain_baseline}
     \end{subfigure}
     \hfill
     \begin{subfigure}[b]{0.32\textwidth}
         \centering
         \includegraphics[width=\textwidth]{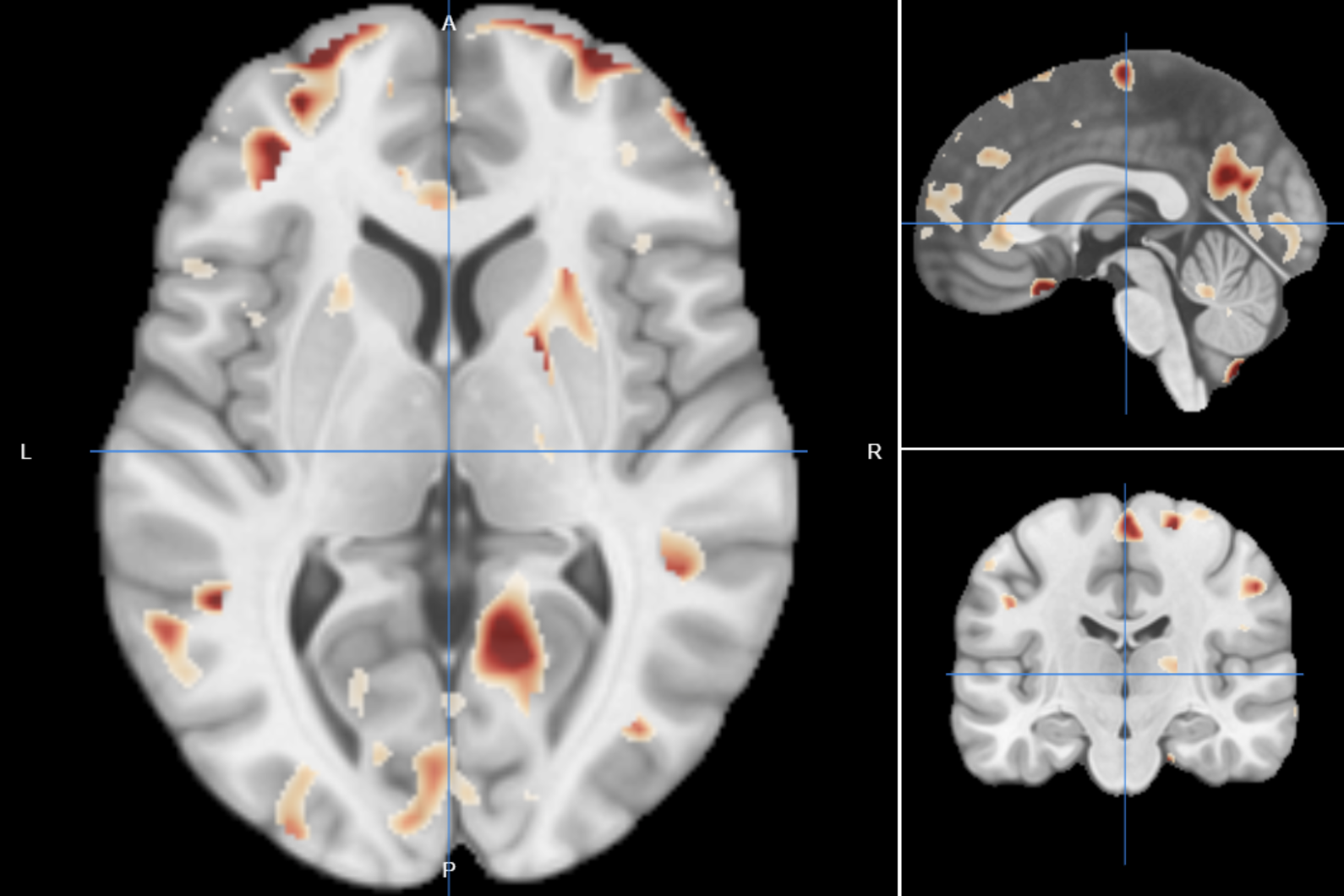}
         \caption{3 Month}
         \label{fig:brain_3m}
     \end{subfigure}
     \hfill
     \begin{subfigure}[b]{0.32\textwidth}
         \centering
         \includegraphics[width=\textwidth]{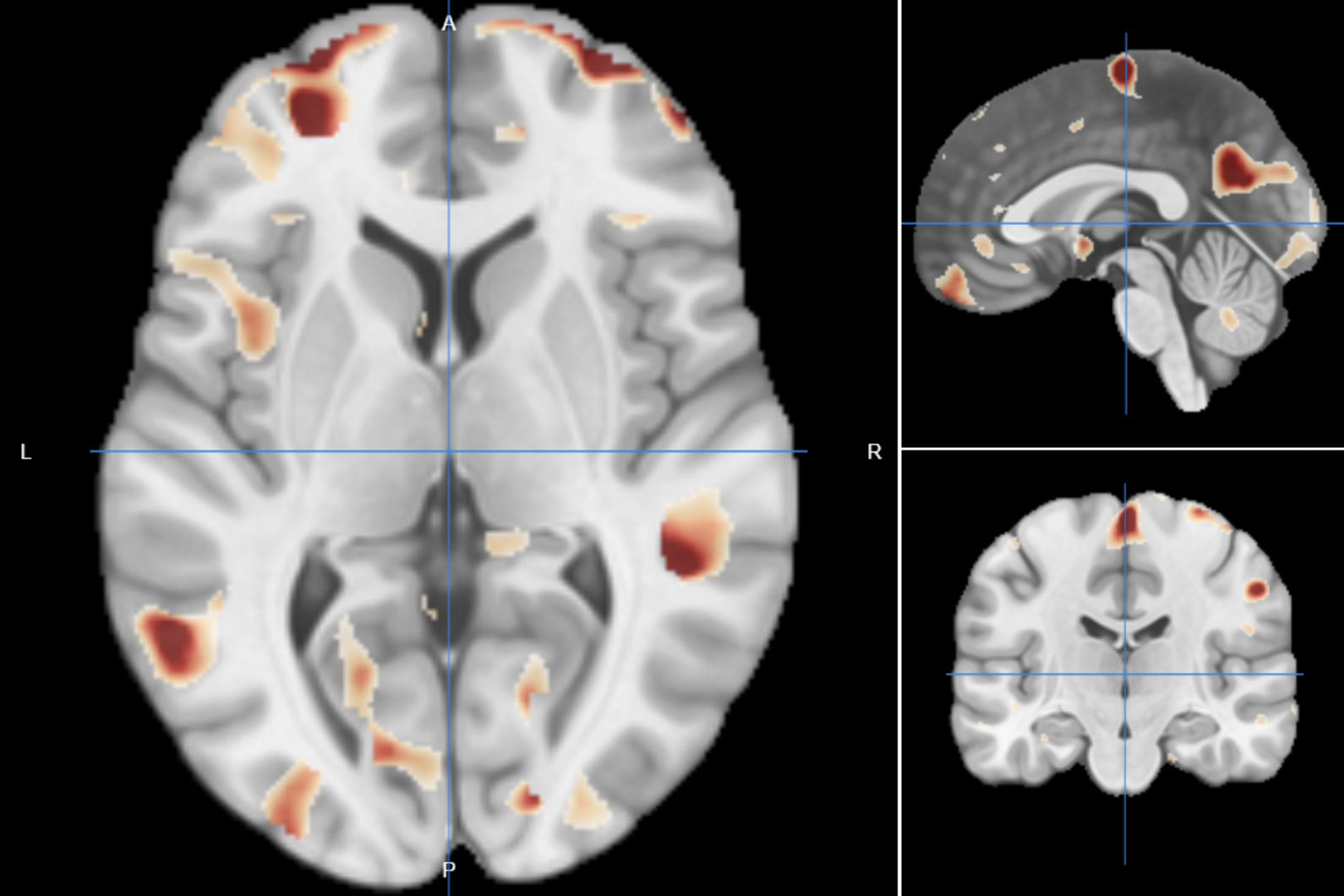}
         \caption{12 Month}
         \label{fig:brain_12m}
     \end{subfigure}
     \caption{Group-level differences in continuous 3D brain activity between antipsychotic-treated and placebo groups in first-episode psychosis patients.}
     \label{fig:brain}
\end{figure}

Traditional approaches struggle in such scenarios. Instead, approximating the distribution of the multidimensional returns and optimizing a function of this distribution offers a promising alternative. This approach can tackle problems with multifaceted rewards and allows for flexible utility functions, capturing a wider range of real-world objectives.


Distributional reinforcement learning emerges as a compelling strategy by considering the entire distribution of rewards. Several algorithms such as HRA \cite{van2017hybrid} and RD2 \cite{lin2020rd2} have been developed for multi-dimensional reward functions. More recently, MD3QN \cite{zhang2021distributional} models the joint distribution of returns and demonstrates that the distributional Bellman operator is a contraction in a maximal form of the Wasserstein distance in finite-dimensional reward spaces.

Our work builds upon this foundation by generalizing the convergence properties of the distributional Bellman operator. Our proposed theory extends its applicability to rewards defined in infinite-dimensional separable Banach spaces, offering remarkable generalizability. It can be effectively applied to finite-dimensional Euclidean spaces as well, a crucial advantage for practical applications. While our proposed computational method primarily focuses on approximating the joint return distribution in Euclidean spaces, we leverage well-established theory demonstrating finite-dimensional Euclidean spaces can effectively approximate higher or infinite-dimensional Banach spaces. Further elaboration will be provided in Section~\ref{s: method1} with proofs in Sections~\ref{s: theory-finite} and ~\ref{s: theory-infinite}.

To quantify the distance between multi-dimensional return distributions, we employ the Wasserstein metric as a theoretical measure to investigate the contraction property of the distributional Bellman operator.
The Wasserstein distance is a robust measure for comparing distributions, applicable to multi-dimensional scenarios, and finds applications in various machine learning areas beyond reinforcement learning, such as Generative Adversarial Networks \cite{arjovsky2017wasserstein} and drug discovery \cite{dai2021drug}. Despite its utility, direct application in multidimensional settings is constrained by computational complexity. To address this, alternatives such as the sliced Wasserstein metric \cite{rabin2012wasserstein, bonnotte2013unidimensional} and max-sliced Wasserstein distance \cite{deshpande2019max} have been proposed, offering computational efficiency and favorable properties \cite{paty2019subspace, bayraktar2021strong}. In this work, we explore the properties of the max-sliced Wasserstein distance (detailed in Section~\ref{s: theory-wass}) and leverage this metric for simulations validating our proposed algorithms.

The article is structured as follows. Section~\ref{s: setting} introduces key assumptions and notations. In Section~\ref{s: method}, we introduce key theorems and propose a reinforcement learning algorithm designed to approximate the distribution of random returns which are finite or infinite-dimensional. This algorithm enables the maximization of any real-valued utility function of these returns. Section~\ref{s: theory} provides rigorous theoretical foundations. We prove three crucial theorems: 1. The distributional Bellman operator is a contraction in a maximal form of the Wasserstein distance when multidimensional reward and return lie in a Banach space. 2. We demonstrate that random variables defined in a Banach space can be approximated with an arbitrarily small error. These theorems provide a strong foundation for our proposed method, enabling it to address a wider range of problems. 3. We investigate the properties of the max-sliced Wasserstein distance which are useful for finite-dimensional Euclidean rewards. In section~\ref{s: simulation}, we present simulation results that validate the effectiveness of our proposed algorithms.

\section{Setting}
\label{s: setting}

Numerous reinforcement learning techniques have been developed to discover an optimal policy that maximizes the expectation of the discounted sum of univariate rewards. However, the objectives of a policy in real-world scenarios are not always confined to the expectation of a one-dimensional return. Our aim is to devise a policy with broader objectives than those of traditional reinforcement learning algorithms. In order to accomplish this, we define the reward and return to exist within a separable Banach space, which includes both separable Hilbert spaces and Euclidean spaces.

 As assumed in many reinforcement learning papers, we consider the interaction between an agent and an environment as a time-homogeneous Markov decision process ($\mathbb{S, A}, R, P, \gamma$). $\mathbb{S}$ and $\mathbb{A}$ are respectively the state and action spaces, a policy $\pi$ maps each state $s \in \mathbb{S}$ to a probability distribution over the action space $\mathbb{A}$. $P$ is the transition kernel $P(\cdot | s, a)$ and $ \gamma \in [0,1]$ is the discount factor. In conventional reinforcement learning algorithms, $R$ and $Z$ typically represent one-dimensional random rewards and the discounted sum of rewards, respectively. These algorithms aim to find the policy that maximizes the expected value of $Z$. 
 
 In this paper, we deviate from this notation. We denote multidimensional random rewards as $R$, explicitly defined within a Banach space $(\mathbb{B}, \|\cdot\|_\mathbb{B})$. Consequently, the random return $Z$, the sum of the discounted rewards, is also a multidimensional random vector defined in the same metric space. In other words, $R$ and $Z$ are $\mathbb{B}$-valued random variables. 
 The objective of this study is to estimate the distribution of multidimensional random returns and determine an optimal policy that maximizes a user-defined utility function. This utility function, denoted as $\phi(Z)$, takes the distribution of the random returns as input and outputs a single real-valued scalar, quantifying the overall desirability of that specific return distribution. For instance, a trading analyst aiming to maximize the first quartile of future assets while prioritizing risk management might employ $\phi(Z)=Q1(Z)$. As an example with 2-dimensioanl reward space, a cancer study focused on maximizing tumor shrinkage while minimizing adverse side effects could utilize $\phi(Z_1, Z_2) = Median(Z_1) + \alpha \cdot P(Z_2 < \delta)$, where $Z_1$ represents tumor reduction, $Z_2$ measures nerve damage, and $\alpha$ and $\delta$ are pre-defined constants.

 Moreover, in scenarios involving high-dimensional data, such as brain imaging studies aiming to enhance brain activity through Alzheimer's treatment, the utility function might operate on an infinite-dimensional space. For example, $\phi(Z) = \|Z\|_\mathbb{B}$ could measure treatment efficacy by quantifying the overall increase in activation levels in key brain regions associated with memory and cognition. Here, $Z$ represents a sequence vector of activation levels across multiple brain regions and time points, constituting an infinite-dimensional space equipped with the norm $\| \cdot \|_\mathbb{B}$.
 
 Additionally, let $R(s_t, a_t)$ represent the immediate random reward when action $a_t$ is taken in a given state $s_t$ at time $t$. The random return $Z^{\pi}(s,a)$ is the sum of discounted rewards by a discount factor $\gamma \in [0,1]$ when an initial action $a$ is taken in a given initial state $s$ and subsequent actions follow policy $\pi$. This multidimensional random vector depends on the transition kernel $P(\cdot | s, a)$ and the policy $\pi$: 
\begin{align*}
Z^{\pi} (s,a) &= \sum_{t=0}^{\infty} \gamma ^{t} R(s_t, a_t)\text{, where }
s_t \sim P(\cdot | s_{t-1}, a_{t-1}), a_t \sim \pi(\cdot|s_t), s_0=s, a_0 = a.
\end{align*}

It's noteworthy that when dealing with one-dimensional reward spaces, the expectation of $Z^{\pi}(s,a)$ is equal to the action-value function commonly employed in various RL algorithms that leverage the Bellman operator. In distributional RL, in contrast, the distribution over returns (i.e., the probability law of $Z$) plays the central role and replaces the traditional value function \cite{bellemare2017distributional, Dabney_Rowland_Bellemare_Munos_2018}. The distribution of random return is referred to as the \textit{value distribution}. Consequently, distributional RL relies on the following key operators:

\begin{enumerate}[label=(\roman*)]
\item Transition operator $P^{\pi} : \mathcal{P}(\mathbb{B})  \rightarrow \mathcal{P}(\mathbb{B})$:
\begin{align*}
P^{\pi}Z(s,a) & :\,{\buildrel D \over =}\, Z(S', A'),\;
           S' \sim P(\cdot | s, a), A' \sim \pi(\cdot|S'),
\end{align*}
\item Distributional Bellman operator $\mathcal{T}^{\pi} : \mathcal{P}(\mathbb{B})  \rightarrow \mathcal{P}(\mathbb{B}) $:
\begin{align*}
\mathcal{T}^{\pi} Z(s,a) :\,{\buildrel D \over =}\, R(s,a)+ \gamma P^{\pi}Z(s,a),
\end{align*}
\item Distributional Bellman optimality operator $\mathcal{T} : \mathcal{P}(\mathbb{B})  \rightarrow \mathcal{P}(\mathbb{B}) $:
\begin{align*}
\mathcal{T}Z(s,a) := \mathcal{T}^{\pi} Z(s,a) \text{ for some } \pi \in \mathcal{G}_{\phi, Z},
\end{align*}
where $\mathcal{G}_{\phi, Z}$ is the greedy policy for a utility function $\phi$ of $Z$ that returns a constant scalar:
\begin{align*}
\mathcal{G}_{\phi, Z} := \Big\{ \pi : \phi \big(Z(s, \pi(\cdot | s))  \big) = \max_{a' \in \mathcal{A}} \phi \big( Z(s, a')\big) \Big\}.
\end{align*}
\end{enumerate}

In Section~\ref{s: contraction_Banach}, we demonstrate that the distributional Bellman operator is a contraction even within the broader context of infinite-dimensional separable Banach spaces. This contraction property in terms of the distribution of $Z$ is the key concept of the algorithm proposed in Section~\ref{s: method}.


\section{Method}
\label{s: method}

This section introduces our proposed reinforcement learning method for estimating the distribution of random returns in high-dimensional reward spaces (represented mathematically as a Banach space). Our approach leverages two key theorems that enable us to approximate any random variable defined in a Banach space with a random variable in a computationally efficient Euclidean space. This approximation technique allows us to address problems in these high-dimensional spaces with practical computational efficiency. We will delve into the details of the theorems and the algorithm in the following subsections.


\subsection{Approximation of value distribution}
\label{s: method1}
As defined in Section~\ref{s: setting}, let $(\mathbb{B}, \|\cdot\|_\mathbb{B})$ be a Banach space, and let $\mathbb{A}$ and $\mathbb{S}$ be action and state spaces, respectively. At each time stage, the current state $S$ and the action $A$ given $S$ are from a common probability space $(\Omega, \mathcal{F}_\Omega, \mathbb{P}_\Omega)$, where $\Omega=\mathbb{A} \times \mathbb{S}$. Additionally, let $\mathcal{H} \subset \mathbb{A} \times \mathbb{S}$ satisfy $P\big( (A,S) \in \mathcal{H}\big) = 1$. Given a realization of state-action pair $p=(s,a)$ and policy $\pi$, the random return $Z^{\pi}(s,a)$ is in a probability space $(\Omega', \mathcal{F}_{\Omega'}, \mathbb{P}_{p,\pi} )$. $Z^{\pi}(s,a)$ is a $\mathbb{B}$-valued random variable, i.e., $Z^{\pi}(s,a)$ is a measurable map from $(\Omega', \mathcal{F}_{\Omega'}, \mathcal{P}_{p,\pi} )$ to $(\mathbb{B}, \mathcal{B}(\mathbb{B}), \mu_{p,\pi})$, where $\mathcal{B}(\mathbb{B})$ is Borel $\sigma$-algebra of $\mathbb{B}$. The probability measure of $Z^{\pi}(s,a)$ is denoted by $\mu_{p,\pi}$ and depends on the state-action pair $p=(s,a)$ and policy $\pi$ while the sample space remains the same for all states, actions, and policies. To simplify the notation, we use $\mu_{p}$ and $Z(s,a)$ instead of $\mu_{p,\pi}$ and $Z^\pi (s,a)$ because the following statements hold for any policy $\pi$. 

To approximate the support of the distribution of $Z(s, a)$, denoted by $\mathbb{B}$, we make the following assumption:

\begin{assumption}
\label{assumption_m1}
$\forall \epsilon>0, \exists \ 0 < r < \infty \ s.t.$
$$\sup_{(s,a) \in \mathcal{H}} E \Big[ \| Z(s,a)\|_\mathbb{B} 1\{ \|Z(s,a)\|_\mathbb{B} >r \}  \Big] \leq \epsilon.$$
\end{assumption}

For any closed set $A_\epsilon\subset\mathbb{B}$ that contains a hyperball $\{z \in \mathbb{B} : \|z\|_\mathbb{B} \leq r\}$ satisfying Assumption~\ref{assumption_m1}, we define $\Tilde{Z_\epsilon}(s,a)$ as the \emph{projection} of $Z(s,a)$ onto $A_\epsilon$. In this context, projection refers to mapping the original value of $Z$ onto the nearest point in $A_\epsilon$ if the value falls outside of $A_\epsilon$. 

Additionally, let $L(\mathbb{B})$ denote the collection of Lipschitz continuous functions mapping from $\mathbb{B}$ to $\mathbb{R}$ with Lipschitz constant $\leq 1$. Under Assumption~\ref{assumption_m1}, the behavior of random return $Z$ is well approximated with a rectangular projection. The proof of the following key theorem will be provided in Section~\ref{s: theory-finite}.

\begin{theorem}
\label{theorem_projection}
    For any fixed $\epsilon>0$ specified in Assumption~\ref{assumption_m1}, and any closed $A_{\epsilon}\subset\mathbb{B}$ as defined above, the projection $\Tilde{Z}_\epsilon$ satisfies 
    $$\sup_{(s,a) \in \mathcal{H}} \sup_{f \in L(\mathbb{B})} E \big| f(Z(s,a)- f(\Tilde{Z}_\epsilon(s,a) \big| \leq 2 \epsilon.$$
\end{theorem}


Our next focus is on an approximation of infinite-dimensional Banach space with projection onto a finite Euclidean space. To facilitate this, we introduce an additional assumption as follows.

\begin{assumption}
\label{assumption_m2}
$\forall \delta>0, \exists$ compact $B_\delta \subset \mathbb{B} \ s.t.$
$$\sup_{(s,a) \in \mathcal{H}} P\big( Z(s,a) \notin B_\delta \big) < \delta. $$
\end{assumption}
This assumption implies that there exists a separable subspace $\mathbb{B}_0 \subset \mathbb{B}$ such that $\inf_{(a,s)\in \mathcal{H}} P(Z(s,a) \in \mathbb{B}_0) =1$. This can be shown by letting $\mathbb{B}_0$ be the closure of $\bigcup_{n \geq 1}B_{1/n}$. This defines a type of uniform separability of the rewards over $\mathcal{H}$.

With this assumption in place, we can establish an approximation theory within the realm of infinite-dimensional Banach spaces. The proof of this is provided in Section~\ref{s: theory-infinite}.

\begin{theorem}
\label{theorem_banach_approximation_ver2}
    Let $\mathbb{B}_0 \subset \mathbb{B}$ be a Banach space with $\mathbb{B}_0$ separable and $\overline{lin}\mathbb{B}_0 \subset \mathbb{B}$. Then for every $\delta >0$, $\exists$ a compact set $A_\delta \subset \mathbb{B}_0$, an integer $k < \infty$, a continuous map $F_\delta (z)$ with domain $\mathbb{B}$ and range $\subset \mathbb{R}^k$, and a Lipschitz continuous function $J_\delta: \mathbb{R}^k \longmapsto \mathbb{B}$ $s.t.$ \\
    $Z_\delta(s,a) = J_\delta F_\delta (Z(s,a))$ and $\Tilde{Z}_\delta(s,a)=\text{the projection of }Z_\delta(s,a) \text{ onto }A_\delta$, such that
    $$\sup_{(s,a) \in \mathcal{H}} \sup_{f \in L(\mathbb{B})} E \Big| f(Z(s,a)) - f(\Tilde{Z}_\delta(s,a))  \Big|< \delta.$$
\end{theorem}

Suppose we have a random variable defined in a Banach space. The fundamental concept of this theorem is that by defining a sufficiently extensive compact set and projecting the random variable onto it, we can effectively approximate the projected variable with a random variable in a finite dimensional Euclidean space. By Theorem~\ref{theorem_projection}, the random variable in Euclidean space can, in turn, be \emph{well approximated} by a rectangular approximation in the Euclidean space. 

To sum up, any random variable in a separable Banach space can be approximated by a random variable contained in a bounded rectangle within a Euclidean space. This capability allows us to tackle problems in infinite or high-dimensional reward spaces with practical computational efficiency. Leveraging this capability, We propose an algorithm that is computationally feasible within Euclidean spaces, utilizing the concept of a hypercube. The algorithm will be detailed in the next section.


\subsection{Algorithm based on distributional Bellman operator} \label{s:method-algo}
In this section, we introduce an algorithm designed to optimize a multi-dimensional return in a distributional reinforcement learning setting, within a robust supporting theoretical foundation, when the return is Euclidean valued or approximately so. Utilizing dynamic programming, we apply the distributional Bellman operator to the algorithm and establish convergence of the estimated value distribution to the true distribution. 

As in the c51 algorithm by \cite{bellemare2017distributional}, we approximate the value distribution with discrete supports. We denote $A_\epsilon$ to be the minimal hypercube containing the ball in Assumption~\ref{assumption_m1}. To construct uniformly distanced bins in $A_\epsilon$, we define mutually exclusive rectangles $K(\mathbf{z_i})$ as follows; \\
$A_\epsilon \supset K = K(z_1) \cup K(z_2) \cdots \cup K(z_N) $, where the $K(\mathbf{z_i})$'s are mutually exclusive half-closed hyber cubes with center $\mathbf{z_i} = (z_{i1}, z_{i2}, \dots )^T $ s.t.
\begin{align*}
    K(\mathbf{z_i}) = \Big\{ \mathbf{w} = (w_1, w_2, \dots)^T : & z_{i1}-c_1 \leq w_1 < z_{i1}+c_1,  \\
    & z_{i2}-c_2 \leq w_2 < z_{i2}+c_2, \\
    &  \dots , \\
    & \mathbf{z_i} = (z_{i1}, z_{i2}, \dots )^T \qquad \Big\}. 
\end{align*}    

Note that the sizes of the $K(\mathbf{z_i})$'s are equal and are determined by the number of categories, $N \in \mathbb{N}$. A large $N$ leads to a precise approximation of $Z(s,a)$. Additionally, let $p_i(s,a)$ denote $P\Big(Z(s,a) \in K(\mathbf{z_i}) \Big)$. Clearly, $\mu_p(A_\epsilon) \geq \mu_p(K) = \mu_p\Big(K(\mathbf{z_1}) \cup K(\mathbf{z_2}) \cup \cdots \cup K(\mathbf{z_N}) \Big)=\mu_p\big( K(\mathbf{z_1}) \big) + \mu_p\big( K(\mathbf{z_2}) \big) + \cdots \mu_p\big( K(\mathbf{z_N}) \big)$. Note that the difference between $A_\epsilon$ and $K$ only comes from the probability associated with the upper boundary of $A_\epsilon$ when increasing the size of $K$ to its maximum. When this probability is zero, the probability measure of $A_\epsilon$ is equivalent to the probability measure of $K$, that is, $\mu_p(A_\epsilon) = \mu_p(K)$ under the condition $\mu_p( \text{upper boundary of } A_\epsilon)=0$. 

When the first action follows a given policy $\pi$ (i.e., $a = \pi(s)$), we represent $Z^\pi(s, a)$ as $Z^\pi(s)$ by omitting the explicit reference to the action $a$. We propose an algorithm to approximate the distribution of $Z^\pi(s)$ based on the distributional Bellman operator as outlined in Algorithm~\ref{algo1}. Let $s$ and $a = \pi(s)$ denote the current state and the action selected according to a policy $\pi$ for a given state $s$. Additionally, let $s' \in \mathbb{S}$, $\Tilde{r} \in \mathbb{B}$, and $\gamma \in [0, 1]$ represent the next stage state, reward, and discount factor, respectively. With a fixed policy $\pi$ and the next stage state $s'$, a return sample $z'$ is drawn from $Z^\pi (s')$. Given $(\Tilde{r}, \gamma, \text{ and } z')$, a sample $z^{updated}$ of $\mathcal{T}^\pi Z^\pi(s)$ can be acquired via $\Tilde{r} + \gamma z'$, emulating the distributional Bellman operator. We repeat this process $n_{sample}$ times and update $Z^\pi(s)$ from the empirical distribution of $z^\pi_{updated}$'s. Note that for a fixed $s$, the distribution of $Z^\pi (s)$ is updated using the current estimated distribution of $Z^\pi (s')$ for every $s' \in \mathbb{S}$. This updating process is iteratively performed $n_{repeat}$ times for each $s\in \mathbb{S}$.

\begin{algorithm}[H]
\caption{Algorithm for a fixed policy $\pi$}\label{algo1}

\Require{$\gamma \in [0, 1] , z_\alpha \text{ for } \alpha = 1, \cdots, N$}
\RepTimes{$n_{repeat}$}{
    \For {$s \in \mathbb{S}$}{
        \For {$i = 1, \cdots, n_{sample}$ }{
            $s' \gets \text{random draw from the transition probability } P(\cdot | s, a = \pi(s))$\\
            $\Tilde{r} \gets \text{random draw from the random reward function } R(s, a=\pi(s), s')$ \\
            $z' \gets \text{random draw from the random return distribution } Z^\pi(s') $ \\
            $z_i^{updated} \gets \underset{z_\alpha}{\arg\min} \|z_\alpha - (\Tilde{r} + \gamma z')\| $ 
        }
        $Z^\pi(s) \gets \text{empirical distribution of } \Big( z_1^{updated}, z_2^{updated}, \cdots, z_{n_{sample} }^{updated} \Big)$
    }
}
\end{algorithm}

In Algorithm~\ref{algo1}, the $z_\alpha$'s are center of half-closed rectangles defined above. Note that the transition probability $P(\cdot | s, a)$ and reward function $R(s, a=\pi(s), s')$ are involved in this algorithm. Therefore, these components need to be estimated from the observed transition data. 

Algorithm~\ref{algo2} illustrates how this estimation is executed within the process of searching for the optimal policy $\pi^{opt}\in\Pi$ with respect to $\phi(Z^\pi(s))$. Upon observing a transition, the transition probability $P(\cdot | s, a)$ and reward function $R(s, a=\pi(s), s')$ are updated. Subsequently, for every policy $\pi \in \Pi$, the distributions of $Z^\pi(s)$ are updated using Algorithm~\ref{algo1}. Finally, utilizing the estimated value distributions, we determine optimal policy $\pi^{opt}$ given the current state $s$.

\begin{algorithm}[H]
\caption{Optimal Policy Search}\label{algo2}
\Require{Policy set $\Pi$, utility function $\phi(\cdot)$}
\Input{Observed transition}
\Update{The transition probability $\hat{P}(\cdot | s, a)$ and the random reward function $\hat{R}(s, a, s')$}
\For {$\pi \in \Pi$}{
        \Update{$Z^\pi(s)$ for every $s \in \mathbb{S}$ by Algorithm~\ref{algo1}} 
    }
$\pi^{opt}(s) \gets \arg\max_{\pi \in \Pi} \phi(Z^\pi(s))$\\
\Return{$\pi^{opt}(s)$}
\end{algorithm}

Our approach stores the distributions of $Z^\pi (s)$ for all possible state $s \in \mathbb{S}$ under the fixed policy and updates these distributions using the defined algorithm. While this method may potentially be computationally inefficient, it effectively estimates the true value distribution. We will provide comprehensive simulation results in Section~\ref{s: simulation} to showcase its performance.
The algorithms presented in this section are founded on the concept of the distributional Bellman operator. In the subsequent section, we will establish a rigorous proof demonstrating that this operator exhibits contraction properties in the Banach space.


\section{Theory}
\label{s: theory}

\subsection{Contraction Property of the Distributional Bellman Operator in Banach Space}
\label{s: contraction_Banach}
The algorithm proposed in Section~\ref{s:method-algo} is based on the distributional Bellman operator. In this section, we aim to prove that the distributional Bellman operator is a contraction in a separable Banach space. 
Let $\mathbb{B}$ be a complete and separable Banach space with norm $\| \cdot \|_{\mathbb{B}}$. We define $\mathcal{P}(\mathbb{B})$ as the space of all Borel probability measures on $\mathbb{B}$, and let $\mathcal{P}_1 (\mathbb{B}) \subset\mathcal{P} (\mathbb{B})$ consist of all $\mu \in \mathcal{P} (\mathbb{B})$ such that $\int \| X \|_{\mathbb{B}} d\mu < \infty$. 
We also introduce the set $BL_1(\mathbb{B})$, which consists of all functions $f : \mathbb{B} \longmapsto \mathbb{R}$ such that $\sup_{x \in \mathbb{B}} |f(x)| \leq 1$ and $\| f \|_l \leq 1$, where 
\begin{align*}
    \| f \|_l = \sup_{x, y \in \mathbb{B} : \| x-y\|_\mathbb{B}>0} \frac{|f(x) - f(y)|}{\| x-y \|_{\mathbb{B}}} .
\end{align*}
Let $BL_1^*(\mathbb{B}) \supset BL_1(\mathbb{B})$ be the set of all functions $f : \mathbb{B} \longmapsto \mathbb{R}$ such that $\| f \|_l \leq 1$.
We define the following norms on elements of $\mathcal{P} (\mathbb{B})$:
\begin{align*}
    &d_{BL_1}(\mu, \nu) = \sup_{f \in BL_1(\mathbb{B})} E_{\mu} f - E_{\nu} f\\
    &\text{and}\\
    &d_{BL_1^*}(\mu, \nu) = \sup_{f \in BL_1^*(\mathbb{B})} E_{\mu} f - E_{\nu} f .\\
\end{align*}
For any $\mu$ and $\nu \in \mathcal{P} (\mathbb{B})$, let $\Gamma(\mu, \nu)$ be the set of all joint probability measures $\lambda \in \mathcal{P} (\mathbb{B}) \times \mathcal{P} (\mathbb{B})$. 
For any measurable set $B \subset \mathbb{B}$, $\lambda(B, \mathbb{B}) = \mu(B)$ and $\lambda(\mathbb{B}, B) = \nu(B)$, i.e., $\lambda$ has marginals $\mu$ and $\nu$. If we let $(X,Y)$ be a random variable with distribution $\lambda \in \Gamma$, then $E_\mu f(X) - E_\nu f(Y) = E_\lambda (f(X) - f(Y))$, in fact, this is trivially true for all joint probability measures in $\Gamma$.

\begin{remark}
    In other words,
    \begin{align*}
    &d_{BL_1}(\mu, \nu) = \sup_{f \in BL_1(\mathbb{B})} E_{\lambda} (f(X) - f(Y)) \text{ for any } \lambda \in \Gamma(\mu, \nu),  \\
    &\text{ and also }\\
    &d_{BL_1^*}(\mu, \nu) = \sup_{f \in BL_1^*(\mathbb{B})} E_{\lambda} (f(X) - f(Y)) \text{ for any } \lambda \in \Gamma(\mu, \nu).\\
    \end{align*}
\end{remark}

The following is the Kantorovich–Rubinstein (K-R) theorem for separable Banach spaces \cite{de1982invariance, edwards2011kantorovich} which we restate here for convenience:
\begin{theorem}[K-R Theorem]\label{theorem:KR}
    For any $\mu, \nu \in \mathcal{P}(\mathbb{B}),$ where $\mathbb{B}$ is a separable Banach space with norm $\| \cdot \|_\mathbb{B}$, 
    \begin{align}
        d_{BL_1^*}(\mu, \nu) = \inf_{\lambda \in \Gamma(\mu, \nu)} E_\lambda \| X-Y \|_\mathbb{B} := W_1(\mu, \nu), \label{1}
    \end{align}
    where $E_\lambda$ is the expectation over $(X,Y) \sim \lambda$. 
\end{theorem}
The norm $W_1$ is the $1$-Wasserstein distance. Note that $BL_1(\mathbb{B}) \subset BL_1^*(\mathbb{B})$, with both sets being closed and $BL_1(\mathbb{B})$ being compact, with closed bounded subsets of $BL_1^*(\mathbb{B})$ being compact.
The following lemma provides several useful properties of the 1-Wasserstein distance and related results:
\begin{lemma}
\label{lemma_A}
Let $\mu, \nu \in \mathcal{P}_1(\mathbb{B})$. Then:
    \begin{enumerate}[label=(\roman*)]
    \item \label{aaa}
    $d_{BL_1}(\mu, \nu) \leq d_{BL_1^*}(\mu, \nu)$ and there exists an $f_1 \in BL_1(\mathbb{B})$ and an $f_2 \in BL_1^*(\mathbb{B})$ such that 
        \begin{align*}
            &\int |f_2| d\mu + \int |f_2| d\nu \leq \int \| X \|_\mathbb{B} d\mu + \int \| Y \|_\mathbb{B} d\nu < \infty, \; \text{ and both }\\
            &E_\lambda[f_1(X) - f_1(Y)] = d_{BL_1}(\mu, \nu) \;\text{and}\; E_\lambda[f_2(X) - f_2(Y)] = d_{BL_1^*}(\mu, \nu), 
        \end{align*}
        for any $\lambda \in \Gamma(\mu, \nu)$.
    \item \label{bbb} $d_{BL_1}(\mu, \nu)$ makes $\mathcal{P}(\mathbb{B})$ into a complete metric space and $d_{BL_1^*}(\mu, \nu)$ makes $\mathcal{P}_1(\mathbb{B})$ into a complete metric space.
    \item \label{ccc} For any $\mu, \nu \in \mathcal{P}_1(\mathbb{B})$, and any $0 \leq \alpha < \infty$, let $\mu_\alpha = $ the probability measure of $\alpha X$, where $X \sim \mu$ and $\nu_\alpha = $ the probability measure of $\alpha Y$, where $Y \sim \nu$. Let $\Gamma(\mu, \nu)$ be the collection of all joint probability distributions for $(X,Y)$ such that $X \sim \mu$ and $Y \sim \nu$. Then
    \begin{align}
        d_{BL_1^*}(\mu_\alpha, \nu_\alpha) = \alpha d_{BL_1^*}(\mu, \nu). \label{1.0}
    \end{align}
\end{enumerate}
\end{lemma}

\begin{proof}
    \begin{enumerate}[label=(\roman*)]
    \item Fix $\mu, \nu \in \mathcal{P}_1(\mathbb{B})$. The first inequality follows directly from the fact that $BL_1(\mathbb{B}) \subset BL_1^*(\mathbb{B})$. Note that $d_{BL_1}$ is bounded below by zero and above by 2. Let $[0,2] \ni c_1 = d_{BL_1}(\mu, \nu)$. Consider a sequence of functions $\{f_n\} \in BL_1(\mathbb{B})$ such that for any choice of $\lambda \in \Gamma(\mu, \nu)$,
     $$E_\lambda [f_n(X) - f_n(Y)] \longrightarrow d_{BL_1}(\mu, \nu).$$
    Since $BL_1(\mathbb{B})$ is compact, there exists a subsequence converging uniformly to some $f_1 \in BL_1(\mathbb{B})$. By design, $f_1$ satisfies
    $$E_\lambda [f_1(X) - f_1(Y)] = d_{BL_1}(\mu, \nu).$$
    Let $B \subset BL_1^*(\mathbb{B})$ be the subset of all $f\in Bl_1^*(\mathbb{B})$ which satisfy $$\int|f|d\mu + \int|f|d\nu \leq \int \| X \|_\mathbb{B} d\mu + \int \| Y \|_\mathbb{B} d\nu < \infty . $$ 
    Note fact that 
    $$d_{BL_1^*}(\mu, \nu) = \inf_{\lambda \in \Gamma(\mu, \nu)} E_\lambda \| X - Y\|_\mathbb{B} \leq \int \| X\|_\mathbb{B} d\mu + \int \| Y\|_\mathbb{B} d\nu < \infty.$$
    
    Now for any $f \in BL_1^*(\mathbb{B})$,
    \begin{align*}
        E_\lambda [f(X) - f(Y)] &= E_\lambda [f(X) - f(0) -  f(Y) + f(0)] \\ 
                               &\leq E_\lambda [\| X \|_\mathbb{B} + \| Y\|_\mathbb{B}] \\
                               &= \int \| X\|_\mathbb{B} d\mu + \int \| Y\|_\mathbb{B} d\nu.
    \end{align*}
    Thus $$\sup_{f \in BL_1^*(\mathbb{B})} E_\lambda [f(X) - f(Y)] = \sup_{f \in B} E_\lambda [f(X) - f(Y)].$$
    It is not difficult to show that $B$ is a compact set. By recycling arguments used above for $d_{BL_1}$, we now have that there is an $f_2 \in B$ such that
    $$E_\lambda [f_2(X) - f_2(Y)] = d_{BL_1^*}(\mu, \nu).$$
    \item The fact that $d_{BL_1}$ makes $\mathcal{P}(\mathbb{B})$ into complete metric space following from part ($\romannumeral 8$) of Theorem 7.6 in \cite{kosorok2008empirical}. Suppose $\{\mu_n\}\in \mathcal{P}_1(\mathbb{B})$ and $\lim_{n \rightarrow \infty} \sup_{m \geq n} d_{BL_1^*}(\mu_n, \mu_m) = 0$.
    This means they become equivalent distributionally. Note also that since both $g:x \mapsto \|X\|_\mathbb{B}$ and $h:x \mapsto -\|X\|_\mathbb{B}$ are in $BL_1^*$, we have that
    $$E_{\lambda_n} \big| \|X\|_\mathbb{B} - \|Y\|_\mathbb{B}\big| \leq d_{BL_1^*}(X, Y)$$
    for any $\lambda_n \in \Gamma(\mu_n, \mu_m)$. Now for any $\epsilon > 0$, $\exists n_0 < \infty:\;\forall m \geq n \geq n_0$, $d_{BL_1^*}(\mu_m, \mu_n) \leq \epsilon$. Thus, for all $m \geq 0$,
    \begin{align*}
        E_{\mu_m}\| X \|_\mathbb{B} &\leq E_{\mu_{n_0}}\| Y \|_\mathbb{B} + E_{\lambda_m^*} \big| \|X\|_\mathbb{B} - \|Y\|_\mathbb{B}\big| \\
                         &\leq E_{\mu_{n_0}}\| Y \|_\mathbb{B} + \epsilon \\
                         & < \infty,
    \end{align*}
    where $\lambda_m^* \in \Gamma(\mu_{n_0}, \mu_m)$. Since $\epsilon$ is arbitrary, we have
    $$\limsup_{n \rightarrow \infty} E_{\mu_n} \| X \|_\mathbb{B} < \infty .$$
    Thus, all limit points of a sequence $\{\mu_n\} \in \mathcal{P}_1(\mathbb{B})$, for the metric $d_{BL_1}$, remain in $\mathcal{P}_1(\mathbb{B})$, and thus $\mathcal{P}_1(\mathbb{B})$ is a complete metric space for the metric $d_{BL_1^*}$.
    \item   If $\alpha = 0$, (\ref{1.0}) holds trivially. 
    Assume $0 < \alpha < \infty$. Then
    \begin{align*}
        d_{BL_1^*}(\mu_\alpha, \nu_\alpha) &= \sup_{f \in {BL_1^*(\mathbb{B})}} E_\mu  f(\alpha X) - E_\nu f(\alpha Y) \\
        &= \sup_{f \in {BL_1^*(\mathbb{B})}} E_\lambda [ f(\alpha X) -  f(\alpha Y)],  
    \end{align*}
    for any $\gamma \in \Gamma(\mu, \nu)$. For a fixed $0<\alpha<\infty$, 
    let $\overset{\sim}{BL_1^*}(\mathbb{B})$ be the set of functions of the form $X \mapsto \alpha ^{-1}f(\alpha X)$ for $f \in BL_1^*(\mathbb{B})$. We now establish the relationship between the two sets, $\overset{\sim}{BL_1^*}(\mathbb{B})$ and $BL_1^*(\mathbb{B})$. 
    
    For any $X, Y \in \mathbb{B}$,
    $$|\alpha^{-1}f(\alpha X) - \alpha^{-1}f(\alpha Y)| \leq \alpha^{-1} \| \alpha X - \alpha Y\|_\mathbb{B} = \| X-Y \|_\mathbb{B}.$$    
    Thus $\overset{\sim}{BL_1^*}(\mathbb{B}) \subset BL_1^*(\mathbb{B})$. Next we show the converse that $BL_1^*(\mathbb{B}) \subset \overset{\sim}{BL_1^*}(\mathbb{B})$. This relationship holds if for any $f \in BL_1^*(\mathbb{B})$, there exists a function $g$ such that $f(X) = \alpha^{-1}g(\alpha X)$ for all $X \in \mathbb{B}$ and $g \in BL_1^*(\mathbb{B})$. Let $f(X) = \alpha^{-1}g(\alpha X)$, then $g(X) = \alpha f(\alpha^{-1}X)$. We can verify that $g \in BL_1^*(\mathbb{B})$:
    \begin{align*}
        |g(X) - g(Y) | &= | \alpha f(\alpha^{-1}X) - \alpha f(\alpha^{-1}Y) | \\
            & = \alpha |  f(\alpha^{-1}X) - f(\alpha^{-1}Y) | \\
            & \leq \alpha \| \alpha^{-1}X - \alpha^{-1}Y \|_\mathbb{B} \\
            & = \| X - Y \|_\mathbb{B} .
    \end{align*}    
    Thus $g(X) \in BL_1^*(\mathbb{B})$ and $\overset{\sim}{BL_1^*}(\mathbb{B}) \subset BL_1^*(\mathbb{B})$. Therefore $\overset{\sim}{BL_1^*}(\mathbb{B}) = BL_1^*(\mathbb{B})$. Hence
    \begin{align*}
        \sup_{f \in BL_1^*(\mathbb{B})} E_\lambda [f(\alpha X) - f(\alpha Y)] 
             &= \alpha \sup_{f \in BL_1^*(\mathbb{B})} E_\lambda [\alpha^{-1} f(\alpha X) - \alpha^{-1} f(\alpha Y)]\\
             &= \alpha \sup_{f \in \overset{\sim}{BL_1^*}(\mathbb{B})} E_\lambda [f(X) - f(Y)]\\
             &= \alpha \sup_{f \in BL_1^*(\mathbb{B})} E_\lambda [f(X) - f(Y)],
    \end{align*}
    and the conclusion follows since $\alpha$ was arbitrary.
    \end{enumerate}
\end{proof}

Let $\mathbb{A}$, $\mathbb{S}$, and $\mathbb{B}$ be complete Banach spaces with respective norms $\|\cdot\|_\mathbb{A}$, $\|\cdot\|_\mathbb{S}$, and $\|\cdot\|_\mathbb{B}$. We assume $\mathbb{A}$ is compact, while $\mathbb{S}$ and $\mathbb{B}_0 \in \mathbb{B}$ are separable. Let $\mathcal{X} \subset \mathbb{A} \times \mathbb{S}$ be closed. We define $Q(S)$ as the initial distribution of states in $\mathbb{S}$ and  the ``behavior'' $b(a|s)$ which is the distribution of action $A = a \in \mathbb{A}$ given the current state $S = s$. Let $P(S' | s,a)$ be the transition probability to next state $S' = s'$ given current state $S=s$ and action $A=a$. Assume $P_Q(S \in \mathbb{S}) = 1$ and for any $s \in \mathbb{S}$, $P_b((A;s) \in \mathcal{X}) = 1$. Also, for any $(a,s) \in \mathcal{X}$, $P_P(S' \in \mathbb{S})=1$. Thus $\mathcal{X}$ does not restrict $S$, but it may restrict $A$ given $S=s$. 

Going forward, we assume that any feasible candidate policy $\pi(a|s)$ satisfies that for every $s \in \mathbb{S}$, $P_\pi\big( (A;s) \in \mathcal{X} \big) = 1$. Note also that $\mathcal{X}$ is separable since both $\mathbb{A}$ and $\mathbb{S}$ are separable. For any $\big(  (a,s), s' \big) \in \mathcal{X} \times \mathbb{S}$, let $\overset{\sim}{R}(s,a,s')$ be the return obtained (as a random variable) for someone whose current state is $s$, who received treatment $a$, and whose next state is $s'$. Assume $P\big( \overset{\sim}{R}(s,a,s') \in \mathbb{B}_0\big) = 1$.

Let $S'(s,a)$ denote a random realization of $P(S'|S=s, A=a)$, and we define the composite reward $R(s,a) = \overset{\sim}{R}(s,a,S'(s,a))$. Let $0 \leq \gamma < 1$ be a discount factor. For any feasible policy $\pi$, we define the state-action value distribution $Z^{\pi}(s,a)$ as:
\begin{align}
    Z^{\pi}(s,a) &= \sum_{j=0}^{\infty}\gamma^j R\big( S_j, A_j\big) \nonumber \\
                 &= R(s,a) + \sum_{j=1}^{\infty}\gamma^j R\big( S_j, A_j \big), \label{2.1}
\end{align}
where $S_{j+1}$ is generated from $P(S_{j+1}|S_j, A_j)$, $A_j \sim \pi(S_j)$, $S_0 = s$, and $A_0 = a$. Thus $Z^\pi\big(s, \pi(s)\big)$ is the total discounted reward for an individual who is treated according to policy $\pi$ perpetually. (\ref{2.1}) can be written recursively as
\begin{align}
    Z^{\pi}(s,a) = R(s,a) + \gamma Z^{\pi}\Big( S'(s,a), \pi\big(S'(s,a)\big) \Big). \label{2.2}
\end{align}

Note that each time $\overset{\sim}{R}(s,a,s')$ is randomly drawn, it is independently drawn. Thus $\overset{\sim}{R}(s,a,s')$, changing over $(a,s) \times s' \in \mathcal{X} \times \mathbb{S}$, is not a stochastic process per se as much as a family of random variables.

However, there is dependence in the $R$ terms in (\ref{2.1}) as follows. Given $(a,s) \in \mathcal{X}$, we draw the next state $S'(s,a)$ and then given $s' = S'(s,a)$, receive the reward $\overset{\sim}{R}(s,a,s')$. Based on the new state $s'$, we draw the next action according to $\pi(a'|s')$, and calling this $A'(s')$. With the drawn action $a'=A'(s')$, we draw another next state according to $S''\big( s', a' \big)$. Now we draw a new reward, $\overset{\sim}{R}(s', a', s'')$, where $s'=S'(s,a), a'=A'(s'), s''=S''(s',a') = S''\Big(S'(s,a), A'\big( S'(s,a) \big) \Big)$. Thus
\begin{align*}
    R(s', a') &= \overset{\sim}{R}\Big( s', \pi(s'), S''\big(s',\pi(s') \big)\Big) \\
              &= \overset{\sim}{R}\Big( s'(s,a), \pi\big(s'(s,a)\big), S''\big(s'(s,a), \pi(s'(s,a)) \big)\Big).
\end{align*}
This leads to the conclusion that $R(s,a)$ and $Z^\pi \Big(S'(s,a), \pi\big(S'(s,a) \big) \Big)$ are dependent through the random variable $S'(s,a)$, but would be independent if conditioned on $S'(s,a)$.

Now Define
\begin{align*}
    Z_j^*(s,a) &= \overset{\sim}{R}\big(s,a,S'(s,a)\big) + \gamma V_j^*\Big( S'(s,a), \pi\big( S'(s,a) \big) \Big)\\
               &= R(s,a) + \gamma V_j^*\Big( S'(s,a), \pi\big( S'(s,a) \big)\Big),
\end{align*}
where $V_j^*(s,a)$, for any $(a,s) \in \mathcal{X}$, is random variable in $\mathbb{B}_0$ with probability $1$.
This yields that $Z_j^*(s,a)$ is also a random variable in $\mathbb{B}_0$ with probability $1$.
We are now ready for the following key result:
\begin{theorem}
    Under the stated conditions,
    \begin{align*}
        &\sup_{(a,s) \in \mathcal{X}} d_{BL_1^*}\big(Z_1^*(s,a), Z_2^*(s,a)\big) \leq \gamma \sup_{(a,s) \in \mathcal{X}} d_{BL_1^*}\big(V_1^*(s,a), V_2^*(s,a)\big).
    \end{align*}
\end{theorem}

\begin{proof}
    \begin{align}
    &\text{For any }(a,s)\in{\cal X}\text{, note that} \nonumber\\
        &d_{BL_1^*}\big(Z_1^*(s,a), Z_2^*(s,a)\big) \nonumber\\
         &= d_{BL_1^*}\Bigg(\overset{\sim}{R}\big(s,a,S'(s,a)\big) +\gamma V_1^*\Big( S'(s,a), \pi\big( S'(s,a) \big)), \nonumber\\
          & \qquad \qquad  \overset{\sim}{R}\big(s,a,S'(s,a)\big) +\gamma V_2^*\Big( S'(s,a), \pi\big( S'(s,a) \big))\Bigg) \nonumber\\
         &=\sup_{f \in BL_1^*(\mathbb{B}_0)} \int_{s'} E \Bigg[ f\Big(\overset{\sim}{R}\big(s,a,s'\big) +\gamma V_1^*\big( s', \pi( s' )\big)\Big) \nonumber\\
          & \qquad \qquad \qquad \quad- f\Big(\overset{\sim}{R}\big(s,a,s'\big) +\gamma V_2^*\big( s', \pi( s' )\big)\Big)\Bigg] dP(s'|a,s) \nonumber\\
         &\leq \int_{s'} \sup_{f \in BL_1^*(\mathbb{B}_0)} E \Bigg[ f\Big(\overset{\sim}{R}\big(s,a,s'\big) +\gamma V_1^*\big( s', \pi( s' )\big)\Big) \nonumber\\
          & \qquad \qquad \qquad \quad - f\Big(\overset{\sim}{R}\big(s,a,s'\big) +\gamma V_2^*\big( s', \pi( s' )\big)\Big)\Bigg] dP(s'|a,s) \nonumber \\
         &\leq \int_{s'} \sup_{f \in BL_1^*(\mathbb{B}_0)} E \Bigg[ f\Big(\gamma V_1^*\big( s', \pi( s' )\big)\Big) - f\Big(\gamma V_2^*\big( s', \pi( s' )\big)\Big)\Bigg] dP(s'|a,s) \label{2.3} \\
         &=  \gamma \int_{s'} \sup_{f \in BL_1^*(\mathbb{B}_0)} E \Bigg[ f\Big(V_1^*\big( s', \pi( s' )\big)\Big) - f\Big( V_2^*\big( s', \pi( s' )\big)\Big)\Bigg] dP(s'|a,s)  \label{2.4} \\
         &\leq  \gamma \sup_{(a,s)\in \mathcal{X}} \sup_{f \in BL_1^*(\mathbb{B}_0)} E \Bigg[ f\Big(V_1^*\big( s, a\big)\Big) - f\Big( V_2^*\big( s, a\big)\Big)\Bigg], \label{2.5}
    \end{align}
    where (\ref{2.4}) follows from \textbf{Lemma~\ref{lemma_A}}, part~(\ref{ccc});    (\ref{2.5}) follows from the fact that
      $$\big\{ (\pi(s'), s') : s' \in \mathbb{S}\big\} \subset \big\{ (a,s) : (a,s) \in \mathcal{X} \big\};$$
    (\ref{2.3}) follows from \textbf{Lemma~\ref{lemma_B}} below; and the conclusion follows since the choice of $(s,a)\in{\cal X}$ on the left-hand-side was arbitrary.\\
 \end{proof}

\begin{lemma}
        \label{lemma_B}
        Let $X,Y,Z$ be random variables on a separable Banach space $\mathbb{B}$, with X independent of $Y$ and $Z$. Then
        $$\sup_{f \in BL_1^*(\mathbb{B})} E \big( f(X+Y) - f(X+Z)\big)  \leq \sup_{f \in BL_1^*(\mathbb{B})} E \big( f(Y) - f(Z)\big).$$
    \end{lemma}
\begin{proof}
    \begin{align*}
        \sup_{f \in BL_1^*(\mathbb{B})} & E \big( f(X+Y) - f(X+Z)\big) \\
        & \leq \int_{X} \sup_{f \in BL_1^*(\mathbb{B})} E \big( f(X+Y) - f(X+Z)\big) dP(X)\\
        & \leq \sup_{f \in BL_1^*(\mathbb{B})} E \big( f(Y) - f(Z)\big),
    \end{align*}
where the last inequality follows from the fact that, for a fixed $X \in \mathbb{B}$, and for any $f \in BL_1^*(\mathbb{B})$, we can easily show that the function $U \mapsto f(X+U)$ is also contained in $BL_1^*(\mathbb{B})$.
\end{proof}

\subsection{Hypercube Approximation Theory}\label{s: theory-finite}
We present the proof of Theorem~\ref{theorem_projection} below. The subsequent section will explore its application to hypercubes in separable Banach spaces.
\begin{proof}
Let $A_\epsilon\subset\mathbb{B}$ be a closed set containing $\{z\in\mathbb{B}: \|z\|_{\mathbb{B}} \leq r\}$. Fix an $f\in L(\mathbb{B})$, where $L(\mathbb{B})$ is the set of Lipschitz continuous functions mapping from $\mathbb{B}$ to $\mathbb{R}$ with Lipschitz constant $\leq 1$. Since $|f(Z)-f(\tilde{Z})|=|f(Z)-f(0)-f(\tilde{Z})+f(0)|$, we can assume without loss of generality for the argments given below that $f(0)=0$. Accordingly, the distance between $f(Z)$ and $f(\Tilde{Z})$ is bounded as follows by the triangle inequality:
\label{proof_projection}
    \begin{align*}
    &E \big|f(Z)-f(\Tilde{Z})\big| \\
    & \leq E\big| f(Z)-f(Z) 1\{Z \in A_\epsilon\} \big|
    + E\big| f(Z) 1\{Z \in A_\epsilon\}-f(\Tilde{Z})1\{Z \in A_\epsilon\} \big| \\
    & \qquad +E\big| f(\Tilde{Z})1\{Z \in A_\epsilon\}-f(\Tilde{Z}) \big|.
    \end{align*}
In the inequality above, each term can be bounded as follows
    \begin{align*}
    & E\big| f(Z)-f(Z) 1\{Z \in A_\epsilon\} \big| \\
    &= E \big| f(Z) 1\{Z \notin A_\epsilon\} \big| \\
    &\leq E \Big[\|Z\|_{\mathbb{B}}1\{Z \notin A_\epsilon\} \Big] \\
    &\leq E \big[\|Z\|_{\mathbb{B}} 1\{\|Z\|_{\mathbb{B}}>r\} \big] .\\
    \end{align*}
Similarly, 
    \begin{align*}
    & E\big| f(\Tilde{Z})-f(\Tilde{Z}) 1\{Z \in A_\epsilon\} \big| \\
    &\leq E \big[ \|\Tilde{Z}\|_\mathbb{B} 1\{Z \notin A_\epsilon\} \big] \\
    &\leq E \big[ \|Z\|_\mathbb{B} 1\{Z \notin A_\epsilon\} \big] \\
    &\leq E \big[ \|Z\|_\mathbb{B} 1\{\|Z\|_\mathbb{B}>r\} \big] .\\
    \end{align*}

Lastly, $E\big| f(Z) 1\{Z \in A_\epsilon\}-f(\Tilde{Z})1\{Z \in A_\epsilon\} \big| =0$ because $Z=\Tilde{Z}$ if $Z \in A_\epsilon$ by construction.

Therefore, it follows that
    \begin{align*}
        &\sup_{(s,a) \in \mathcal{H}} \sup_{f \in L(\mathbb{B})} E \Big|  f(Z(s,a)) -f(\Tilde{Z}(s,a)) \Big| \\
        &\leq 2 \sup_{(s,a) \in \mathcal{H}} \sup_{f \in L(\mathbb{B})} E \big[ \|Z\|_\mathbb{B} 1\{\|Z\|_\mathbb{B}>r\} \big] \\
        &\leq 2 \epsilon,
    \end{align*}
and the desired conclusion follows.
\end{proof}

\subsection{Approximation Theory in Banach Space}\label{s: theory-infinite}
In this section, we provide the proof of Theorem~\ref{theorem_banach_approximation_ver2}. Prior to that, we explore the practical application of the theorem. This theorem guarantees the existence of an approximation function, $F_\delta (z):\mathbb{B} \mapsto \mathbb{R}^k$. The key insight is that, despite a high-dimensional random variable potentially having infinitely many dimensions, the approximation can be constructed using a finite-dimensional random variable. The specific construction of $F_\delta (z)$ is not provided by the theorem itself. We now present several examples of settings where we can apply this approximation. The first example is for Euclidean spaces. We then give an example of a separable Hilbert space followed by separable Banach space example.

\begin{example}
  In the Euclidean space $\mathbb{B}=\mathbb{R}^k$ for $1\leq k<\infty$, the construction of the approximation function $F_\delta (z)$ is trivial. We simply set $F_\delta(z)=z$ for any $z$ in the compact set $A$. The primary challenge becomes identifying a suitable compact set $A$ for the approximation.

  Assuming $Z(s,a)\in\mathbb{R}^k$ almost surely for all $(s,a)\in{\cal X}$, we fix $\epsilon>0$ and choose $r>1$ to satisfy Assumption~\ref{assumption_m1} for this $\epsilon$. Now let $A_{\epsilon}$ be the smallest, closed, axis-aligned hypercube containing the $\mathbb{R}^k$-ball of radius $r$. Then $A_{\epsilon}$ is compact since all bounded closed sets are also compact in $\mathbb{R}^k$. Not only that, $A_\epsilon$ trivially satisfies the conditions of Theorem~\ref{theorem_projection}, ensuring an accurate approximation of the distribution of $Z(s,a)$. Moreover, Assumption~\ref{assumption_m2} is easily satisfied for $\delta=\epsilon$ and $B_\delta=A_\epsilon$.

\end{example}

\begin{example}
    Let $\mathbb{B}=H$, where $H$ is a Hilbert space with inner product $\langle \cdot,\cdot\rangle$ and norm $\|\cdot\|$, and assume $H_0\subset H$ is separable. Due to the separability, any element $h\in H_0$ can be expressed as a linear combination of a countable orthonormal basis $\{e_j\}$. Mathematically, this is represented as: $$h=\sum_{j=1}^\infty \langle h,e_j \rangle e_j.$$ 
    Now, let $A \subset H_0$ be a compact set. For any $\delta>0$, we can find a finite integer $k$ such that $$\sup_{h \in A} \left\|\sum_{j=k+1}^{\infty} \langle h,e_j \rangle e_j \right\|< \delta.$$
    Based on this, we define a truncation function $g:H_0 \longmapsto T_k$ as $$g(h)=\sum_{j=1}^{k} \langle h,e_j \rangle e_j.$$ 
    This truncation effectively approximates the original sequences in the compact set $A$. It can be easily shown that: $$\sup_{h \in A} \|g(x)-x\| < \delta .$$

    We can also invoke the Karhunen–Lo\`eve expansion as follows. For all $(s,a)\in{\cal X}$, assume $Z(s,a)\in H$, define $\mu(s,a)=EZ(s,a)$, and assume $Z(s,a)-\mu(s,a) \in H_0$, almost surely. Now fix $(s,a)\in{\cal X}$. By the Karhunen-Lo\`eve expansion theorem, there exists infinite sequences of scalar eigenvalues $\infty>\lambda_1\geq\lambda_2\geq\cdots\geq 0$ and orthonormal eigenfunctions $\{\psi_j,j\geq 1\}\in H_0$ so that 
    $$Z(s,a) = \mu(s,a)+ \sum_{j=1}^{\infty}\lambda_j U_j\psi_j,$$
    almost surely, where $\sum_{j=1}^{\infty}\lambda_j^2<\infty$, and $U_1,U_2,\ldots$ are mean zero random variables with $EU_j^2=1$ and $E[U_jU_k]= 0$ when $j\neq k$, for all $1\leq j,k<\infty$. Moreover, the eigenfunctions $\psi_1,\psi_2,\ldots$ form a basis for a separable Hilbert space $H_1\subset H_0$ wherein $Z(s,a)-\mu(s,a)\in H_1$ almost surely. We note that each $\psi_j$ can be written as a weighted sum of the $\{e_1,e_2,\ldots\}$, but they are otherwise not connected per se.

    This expansion allows us to approximate $Z(s,a)-\mu(s,a)$ with a finite sum by truncating the series at a specific term $k$:
    $$Z_k(s,a) = \mu(s,a)+ \sum_{j=1}^{k}\lambda_j U_j\psi_j.$$
    Let us further assume that the sequence of basis functions in the expansion are the same across all values of $(s,a)\in{\cal X}$; that for each $(s,a)\in{\cal X}$ we have the representation 
    $$Z_k(s,a) = \mu(s,a)+ \sum_{j=1}^{k}\lambda_j(s,a) U_j(s,a)\psi_j,$$
    where $EU_j(s,a)=0$, $EU_j^2(s,a)=1$ and $EU_j(s,a)U_k(s,a)=0$ for all $j\neq k$; and that for every $\epsilon>0$ there exists an integer $1\leq k<\infty$ such that
    $$\sup_{(s,a)\in{\cal X}}\sum_{j=k+1}^\infty \lambda_j^2(s,a)<\epsilon.$$
    Now we can, with arbitrary level of accuracy, approximate $Z(s,a)$ for each $(s,a)\in{\cal X}$ as
$$Z_k(s,a)=\mu(s,a)+\sum_{j=1}^k\lambda_j(s,a)U_j(s,a)\psi_j,$$
which is linear function of the Euclidean random vector 
$$\{\mu(s,a),\lambda_1(s,a)U_1(s,a),\lambda_2(s,a)U_2(s,a),\ldots\}^T\in\mathbb{R}^{k+1}.$$
\end{example}

\begin{example}
    Let $\mathbb{B}=\ell^\infty(T)$, where $\ell^\infty(T)$ denotes the set of all bounded real-valued functions defined on a set $T$, and suppose a random variable $Z\in\ell^\infty(T)$ is separable. This means that, for some metric $\rho$ on $T$ making $(T,\rho)$ totally bounded, $Z\in UC(T,\rho)$ almost surely, where $UC(T,\rho)\subset\ell^\infty(T)$ is the space of all $\rho$-equicontinuous and bounded functions $f:T\mapsto\mathbb{R}$. In other words, if $x \in UC(T, \rho)$, then $\lim_{\delta \downarrow 0} \sup_{t,u \in T:\rho(t,u)<\delta} |x(t)-x(u)|=0$. Furthermore, any compact subset $K \subset UC(T,\rho)$ must satisfy
    \begin{align*}
        & \lim_{\delta \downarrow 0} \sup_{x \in K} \sup_{t,u \in T:\rho(t,u)<\delta} |x(t)-x(u)|=0 \\
        & \sup_{x \in K} \sup_{t \in T} \| x(t) \| < \infty.   
    \end{align*}
    Now suppose that $Z(s,a)\in\ell^{\infty}(T)$ almost surely for all $(s,a)\in{\cal S}$, and suppose that Assumptions~\ref{assumption_m1} and ~\ref{assumption_m2} are satisfied in this instance for $\mathbb{B}=\ell^{\infty}(T)$. Then the compact sets $B_\delta$ must satisfy the conditions given above for some semimetric $\rho$ making $T$ totally bounded.
    
    This means that for every $\epsilon>0$, there exists a finite subset $T_{\epsilon}=\{t_1,\ldots,t_k\}\subset T$, for $1\leq k<\infty$, such that for all $x\in B_{\delta}$, 
    $$\sup_{t\in T}\inf_{s\in T_\epsilon}|x(s)-x(t)|\leq\epsilon.$$
    Furthermore, we can partition $T$ into $k$ non-intersecting subsets $T_1,\ldots,T_k$ so that $t_j\in T_j$, for all $1\leq j\leq k$, and $T=\cup_{j=1}^k T_j$. For any $x\in\ell^{\infty}(T)$, define the approximation $\tilde{x}$ such that for all $t\in T$, $\tilde{x}(t)=\sum_{j=1}^k x(t_j)1\{t\in T_j\}$. By construction, we have that for any $x\in B_\delta$, $\sup_{t\in T}|x(t)-\tilde{x}(t)|\leq\epsilon$. 

    Now we connect this to Theorem~\ref{theorem_banach_approximation_ver2}. Under the above conditions, we first fix $\delta>0$ and let $\epsilon=\delta/4$. From Assumption~\ref{assumption_m1}, there is an $r<\infty$ so that
    $$\sup_{(s,a)\in{\cal X}}E\left[\|Z(s,a)\|_\mathbb{B}1\left\{\|Z(s,a)\|_\mathbb{B}>r\right\}\right]\leq \epsilon/2.$$
    From Assumption~\ref{assumption_m2}, there is a compact $B_\epsilon\subset\ell^\infty(T)$ so that
    $$\inf_{(s,a)\in{\cal X}}P(Z(s,a)\in B_\epsilon^c)\leq \epsilon/(r\vee 1).$$ 
    In the above, we make $r$ slightly larger if necessary to ensure that 
    $B_\epsilon\subset\{x\in\mathbb{B}:\;\|x\|_\mathbb{B}\leq r\}$. 
    Take $T_\epsilon$ as above, and let
    $Z_\delta(s,a)$ be the projection of $Z(s,a)$ onto $B_\epsilon$ and let 
    $$\tilde{Z}_\delta(s,a)(t)=\sum_{j=1}^k Z_\delta(s,a)(t_j)1\{t\in T_j\}.$$

    By construction, this approximation now satisfies
    $$\sup_{(s,a)\in {\cal X}}\sup_{f\in L(\mathbb{B})}E\left|f(Z(s,a))-f(\tilde{Z}_\delta(s,a))\right|<\delta.$$
    In this case, the map $F_\delta:z\mapsto \mathbb{R}^k$ from Theorem~\ref{theorem_banach_approximation_ver2} takes the simple form
    $F_\delta(z)=\left(z(t_1),\ldots,z(t_k)\right)^T$.
\end{example}

As a preliminary step in proving Theorem~\ref{theorem_banach_approximation_ver2}, we introduce a key theorem from \cite{kosorok2008empirical}, which will be directly used in the proof.
 
\begin{theorem}
\label{theorem_banach_approximation}
    Let $\mathbb{B}_0 \subset \mathbb{B}$ be a Banach space with $\mathbb{B}_0$ separable and $\overline{lin}\mathbb{B}_0 \subset \mathbb{B}$. Then for every $\delta >0$ and every compact $A \subset \mathbb{B}_0$, there exists an integer $k < \infty$, elements $y_1, \dots, y_k \in C[0,1]$, continuous functions $f_1, \dots, f_k : \mathbb{B}_0 \longmapsto \mathbb{R}$, and a Lipschitz continuous function $J: \overline{lin}(y_1, \dots, y_k) \longmapsto \mathbb{B}$, such that the map $x \longmapsto T_\delta (x):=J\big( \sum_{j=1}^{k} y_j f_j (x) \big)$ has domain $\mathbb{B}$ and range $\subset \mathbb{B}$, is continuous, and satisfies $\sup_{x \in A} \| T_\delta (x) -x \|_\mathbb{B} < \delta$.  
\end{theorem}

This theorem establishes the existence of an approximation function, $T_\delta (x)$, for high-dimensional vectors residing in a Banach space $\mathbb{B}$. Notably, even though $\mathbb{B}$ might have infinitely many dimensions, the approximation can be constructed using a finite number of terms based on continuous functions defined on $\mathbb{B}_0$. In the remainder of this subsection, we prove Theorem~\ref{theorem_banach_approximation_ver2}. 

\begin{proof}
    Under Assumption~\ref{assumption_m1}, there exists $r$ $s.t.$  $\sup_{(s,a) \in \mathcal{H}} E \Big[ \| Z(s,a)\|_\mathbb{B} 1\{ \|Z(s,a)\|_\mathbb{B} >r \}  \Big] < \delta/6$. Additionally, there exists a compact set $B_\delta$ $s.t.$ $\sup_{(s,a) \in \mathcal{H}} P\big( Z(s,a) \notin B_\delta \big) < \delta/(6r)$ under Assumption~\ref{assumption_m2}. Without loss of generality, we can assume that $0\in B_\delta$ since adding the zero does not change the compactness of $B_{\delta}$. When letting $A_\delta = B_\delta \cap \{z \in \mathbb{B}: \|z\|_\mathbb{B} \leq r \}$, there exist
     a continuous map $F_\delta (z)$ with domain $\mathbb{B}$ and range $\subset \mathbb{R}^k$, a Lipschitz continuous function $J_\delta: \mathbb{R}^k \longmapsto \mathbb{B}$, and 
    a map $T_\delta = J_\delta F_\delta(z)$ $s.t.$ $\sup_{z \in A_\delta} \|T_\delta(z)-z \|_\mathbb{B} < \delta/3$ by Theorem~\ref{theorem_banach_approximation}.
    \begin{align*}
        & E \Big| f(Z) - f(\Tilde{Z}_\delta)  \Big| \\
        & \leq E \Big\| Z - \Tilde{Z}_\delta  \Big\|_\mathbb{B} \text{ when $f$ is Lipschitz continuous} \\
        & = E \Big[ \big\|Z - \Tilde{Z}_\delta \big\|_\mathbb{B} 1\{Z \in A_\delta\} \Big]  
            + E \Big[ \big\|Z - \Tilde{Z}_\delta \big\|_\mathbb{B} 1\{Z \notin A_\delta\} \Big]\\
        & \leq E \Big[ \big\|Z - \Tilde{Z}_\delta \big\|_\mathbb{B} 1\{Z \in A_\delta\} \Big]  
            + E \Big[ \big\|Z\big\|_\mathbb{B} 1\{Z \notin A_\delta\} \Big] 
            + E \Big[ \big\|\Tilde{Z}_\delta  \big\|_\mathbb{B} 1\{Z \notin A_\delta\} \Big]. 
    \end{align*}
    Furthermore, the second term can be bounded as follows:
    \begin{align*}
        & E \Big[ \big\|Z\big\|_\mathbb{B} 1\{Z \notin A_\delta\} \Big] \\
        & = E \Big[ \big\|Z\big\|_\mathbb{B} 1\{\|Z \|_\mathbb{B} \leq r\} 1\{Z \notin B_\delta\} \Big] +
            E \Big[ \big\|Z\big\|_\mathbb{B} 1\{\|Z \|_\mathbb{B} > r\} \Big] \\
        & \leq E \Big[ r 1\{\|Z \|_\mathbb{B} \leq r\} 1\{Z \notin B_\delta\} \Big] +
            E \Big[ \big\|Z\big\|_\mathbb{B} 1\{\|Z \|_\mathbb{B} > r\} \Big] \\
        & \leq E \Big[ r 1\{Z \notin B_\delta\} \Big] +
            E \Big[ \big\|Z\big\|_\mathbb{B} 1\{\|Z \|_\mathbb{B} > r\} \Big] \\
        & = r \cdot P (Z \notin B_\delta) + E \Big[ \|Z\|_\mathbb{B} 1\{ \|Z\|_\mathbb{B}>r\} \Big],\\   
    \end{align*}
    and the third term as follows:
    \begin{align*}
        & E \Big[ \big\|\Tilde{Z}_\delta  \big\|_\mathbb{B} 1\{Z \notin A_\delta\} \Big] \\
        & \leq E \Big[ r 1\{Z \notin A_\delta\} \Big] \text{ (since $\Tilde{Z}_\delta$ is a projection onto } A_\delta \ni 0) \\
        & = E \Big[ r 1\{Z \notin B_\delta\} \Big] + E \Big[ r 1\{Z \in B_\delta\} 1\{ \|Z\|_\mathbb{B}>r\} \Big]\\
        & \leq E \Big[ r 1\{Z \notin B_\delta\} \Big] + E \Big[ \|Z\|_\mathbb{B} 1\{Z \in B_\delta\} 1\{ \|Z\|_\mathbb{B}>r\} \Big]\\
        & \leq r \cdot P (Z \notin B_\delta) + E \Big[ \|Z\|_\mathbb{B} 1\{ \|Z\|_\mathbb{B}>r\} \Big].
    \end{align*}
    
    Based on the aforementioned inequalities, it follows that:
    \begin{align*}
        &\sup_{(s,a) \in \mathcal{H}} \sup_{f \in L(\mathbb{B})} E \Big| f(Z(s,a)) - f(\Tilde{Z}_\delta(s,a))  \Big|\\
        &\leq \sup_{(s,a) \in \mathcal{H}} \Bigg\{  
              E \Big[ \big\|Z(s,a) - \Tilde{Z}_\delta(s,a) \big\|_\mathbb{B} 1\{Z(s,a) \in A_\delta\} \Big]  \\
        & \qquad \qquad + r \cdot P \big(Z(s,a) \notin B_\delta \big) 
          + E \Big[ \|Z(s,a)\|_\mathbb{B} 1\{ \|Z(s,a)\|_\mathbb{B}>r\} \Big] \Bigg\}.
    \end{align*}
    Furthermore, we can establish that each of these terms is bounded by a constant. Specifically
    \begin{align*}
        &\sup_{(s,a) \in \mathcal{H}} E \Big[ \big\|Z(s,a) - \Tilde{Z}_\delta(s,a) \big\|_\mathbb{B} 1\{Z(s,a) \in A_\delta\} \Big] \\
        &\leq \sup_{(s,a) \in \mathcal{H}} \sup_{Z(s,a) \in A_\delta} E \big\|Z(s,a) - \Tilde{Z}_\delta(s,a) \big\|_\mathbb{B} \\
        &\leq \delta /3
    \end{align*}
    and
    $$ r \cdot \sup_{(s,a) \in \mathcal{H}} P \big(Z(s,a) \notin B_\delta \big) \leq \delta/3.$$
    Additionally
    $$ \sup_{(s,a) \in \mathcal{H}} E \Big[ \|Z(s,a)\|_\mathbb{B} 1\{ \|Z(s,a)\|_\mathbb{B}>r\} \Big] \leq \delta/3.$$
    Therefore, it follows that
    $$\sup_{(s,a) \in \mathcal{H}} \sup_{f \in L(\mathbb{B})} E \Big| f(Z(s,a)) - f(\Tilde{Z}_\delta(s,a))  \Big| \leq \delta /3 + \delta /3 + \delta /3 = \delta.$$
        
\end{proof}

\subsection{Convergence and Wasserstein Metric in Banach Spaces}\label{{s: theory-banachwass}}
In this section, we present one theorem which illuminates the strong type of weak convergence implied by convergence with respect to the Wasserstein metric for random variables defined on a separable Banach space $(\mathbb{B}, \| \cdot \|_\mathbb{B})$. We focus on the Wasserstein metric $W_1$, defined as:
$$W_1(X,Y) = \inf_{\lambda \in \Gamma(X, Y)} \int \|x-y \|_\mathbb{B} \lambda(x,y) \,dx \,dy,$$
where $\Gamma(X, Y)$ denotes the set of all couplings $X$ and $Y$.
Additionally, we define $D_1(\mathbb{B})$ to be the set of all continues functions $f:\mathbb{B}\mapsto\mathbb{R}$ such that
$$\sup_{z\in\mathbb{B}}\frac{|f(z)|}{1\vee\|z\|_\mathbb{B}}<\infty.$$

In the following theorem, we will denote the Banach norm $\| \cdot \|_\mathbb{B}$ as $\| \cdot \|$ for notational convenience. Here is the theorem:
\begin{theorem} \label{new.convergence}
Let $(\mathbb{B},\|\cdot\|)$ be a Banach space, and $X_n,X\in\mathbb{B}$ be separable random variables with $E\|X_n\|<\infty$ for all $n\geq 1$ large enough. Then the following are equivalent:
\begin{enumerate}[label=(\roman*)]
\item \label{new.convergence.i} $W_1(X_n,X)\rightarrow 0$, as $n\rightarrow\infty$; 
\item \label{new.convergence.ii} $Ef(X_n)\rightarrow Ef(X)$, as $n\rightarrow\infty$, for all $f\in D_1(\mathbb{B})$. 
\end{enumerate}
If either \textnormal{(\ref{new.convergence.i})} or \textnormal{(\ref{new.convergence.ii})} is true, then $E\|X\|<\infty$; $\lim_{n\rightarrow\infty}E\|X_n\|=E\|X\|$;
$$\lim_{k\rightarrow\infty}\limsup_{n\rightarrow\infty}E\left[\|X_n\|1\{\|X_n\|>k\}\right]=0;$$
and, for every $\epsilon>0$, there exists a compact $A\subset\mathbb{B}$ such that $\liminf_{n\rightarrow\infty}P(X_n\in A^{\delta})\geq 1-\epsilon$, for all $\delta>0$, where $A^\delta = \{x \in \mathbb{B}: \|x-y\| < \delta \text{ for some }y \in A \}$ is the $\delta$-enlargement of $A$.
\end{theorem}
Before giving the proof, we note that conclusion (\ref{new.convergence.ii}) implies and is stronger than convergence in distribution in that it includes convergence in expectation of certain types of unbounded function, including the norm $\|\cdot\|$.

\begin{proof}
First, assume~(\ref{new.convergence.i}). Let $\mu_n$ be the distribution of $X_n$ and $\nu$ be the distribution of $X$. Then, by Theorem~\ref{theorem:KR}, we have that for some $\lambda_n\in\Gamma(\mu_n,\nu)$, $E_{\lambda_n}\|X_n-X\|\rightarrow 0$. This means that
\begin{eqnarray*}
E\|X\|&=&E_{\lambda_n}\|X\|\\
&\leq&E_{\lambda_n}\|X_n\|+E_{\lambda_n}\|X_n-X\|\\
&=&E\|X_n\|+E_{\lambda_n}\|X_n-X\|\\
&<&\infty,
\end{eqnarray*}
for some $1\leq n<\infty$ large enough, by taking the limit over $n$. This implies the $E\|X\| \leq \infty$ when (\ref{new.convergence.i}) is true.

Now, for every $0<k<\infty$, define the map $[0,\infty)\ni x\mapsto$ $h_k(x)=(x-k)_{+}\wedge 1$, 
where $(a)_{+}$ is the positive part of $a$ and $a\wedge b$ is the minimum of $a$ and $b$. 
Now fix an $f\in D_1(\mathbb{B})$. For any $0<k<\infty$, it is not hard to see that 
$X\mapsto f_k(X) = f(X)\big( 1-h_k(\|X\|)\big)$ is bounded and continuous, since $1-h_k(\|X\|)=0$ whenever $\|X\| \geq k+1$, and thus
\begin{align*}
    |f_k(X)| &= \frac{|f(X)|}{\|X\| \vee 1} \big( 1- h_k(\|X\|)\big) \big( \|X\|\vee 1 \big)\\
    & \leq c_1 \big( 1- h_k(\|X\|)\big) \big( \|X\|\vee 1 \big) \text{ for some constant }c_1\\
    & \leq c_1 (k+1) \\
    & < \infty,
\end{align*}
and also since both $f(X)$ and $h_k(X)$ are continuous. 

Thus, since $X_n \rightsquigarrow X$, we have that 
\begin{align}
    E f_k(X_n) \rightarrow E f_k (X), \quad \forall k < \infty. \label{4.4}
\end{align}

Now, fix $\epsilon > 0$. Since $E \| X\| < \infty$, $\exists k_1>1: E \Big[ \| X \| 1\{\|X\|>k_1\}\Big] \leq \epsilon $. Additionally, for all $X \in \mathbb{B}$,
\begin{align*}
    |f(X)| \cdot h_{k_1}(\|X\|) &= \frac{|f(X)|}{\|X\| \vee 1}  h_{k_1}(\|X\|) \big( \|X\|\vee 1 \big) \\
    & \leq c_1 1\{\|X\|>k_1\} \big( \|X\|\vee 1 \big) \\
    & = c_1 1\{\|X\|>k_1\} \|X\|.    
\end{align*}
Therefore, 
\begin{align}
    E \Big[ |f(X)| \cdot h_{k_1}(\|X\|)\Big] \leq c_1 E\Big[ 1\{\|X\|>k_1\} \|X\|\Big] \leq c_1 \epsilon. \label{4.5a}
\end{align}

Now let $k_1<k_2<\infty$, and note that for any $\lambda_n \in \Gamma(X_n, X)$,
\begin{align*}
    E_{\lambda_n} \Big[ \|X_{n}\| 1\{\|X_n\| > k_2 \}\Big]
    &= E_{\lambda_n} \Big[ \|X_{n}\| 1\{\|X_n\| > k_2 \}1\{\|X\|\leq k_1 \} \\
    & \qquad \quad + \|X_{n}\| 1\{\|X_n\| > k_2 \} 1\{\|X\| > k_1 \}\Big] \\
    & \leq E_{\lambda_n} \Big[ (k_1 + \|X_{n} - X\| )1\{ \|X_n\| > k_2\}\Big]  \\
    & \quad + E_{\lambda_n} \Big[ \|X_{n}\|1\{\| X\| > k_1\}\Big]\\
    & \leq E_{\lambda_n} \| X_n\|(k_1/k_2)
    +E_{\lambda_n}\|X_{n} - X\| \\
    & \quad  +E_{\lambda_n}\|X\|1\{\|X\|>k_1\}+E_{\lambda_n} \|X_{n} - X\| \\
    & \leq E_{\lambda_n} \| X_n\|(k_1/k_2) +E_{\lambda_n} \|X\|1\{\|X\|>k_1 \} \\
    & \quad + 2 E_{\lambda_n} \|X_{n} - X\| \\
    & \leq E \| X_n\|(k_1/k_2) + \epsilon + 2 E_{\lambda_n} \|X_{n} - X\|.
\end{align*}
By taking the $\limsup$ of both sides, we have
\begin{align*}
    & \limsup{E_{\lambda_n} \Big[ \|X_{n}\| 1\{\|X_n\| > k_2 \}\Big]} \\
    & \leq \limsup{E \| X_n\|(k_1/k_2)} + \epsilon + 2 \limsup{E_{\lambda_n} \|X_{n} - X\|} \\
    & \leq \limsup{E \| X_n\|(k_1/k_2)} + \epsilon + 2 \limsup{W_1(X_{n}, X)}, \\
    & \qquad \text{by taking the infimum over } \lambda_n \in \Gamma(X_n, X).
\end{align*}
Since it is easy to verify by recycling previous arguments that $\limsup{E\|X_n\|} = c_2 $ for some $c_2<\infty$, we can take $k_2 = (k_1/\epsilon) \vee k_1$ to obtain that
$$ \limsup{E\Big[ \|X_n\| 1 \{ \|X_n\| > k_2 \} \Big]} \leq (1+c_2)\epsilon  $$
and also that
\begin{align}
    \limsup E \Big( |f(X_n)| \cdot h_{k_2}(\|X_n\|)\Big) 
    \leq c_1 \limsup{E\Big[ \|X_n\| 1 \{ \|X_n\| > k_2 \} \Big]} 
    \leq c_1 (1+c_2)\epsilon. \label{4.6}
\end{align}
Thus 
\begin{align*}
    \limsup_{n\rightarrow\infty}Ef(X_n) &= \limsup_{n\rightarrow\infty}\left(E\Big[ f(X_n)\big( 1- h_{k_2}(\|X_n\|)\big) \Big] + E\Big[ f(X_n) h_{k_2} (  \|X_n\| ) \Big]\right)\\
    & \leq E\Big[ f(X)\big( 1- h_{k_2}(\|X\|)\big) \Big] + c_1(1+c_2)\epsilon \text{ by (\ref{4.4}) and (\ref{4.6}) }\\
    &\leq Ef(X) + E \Big[ |f(X)| h_{k_2}(\|X\|) \Big] + c_1(1+c_2)\epsilon \\
    &\leq Ef(X)+c_1(2+c_2)\epsilon\;\text{ by (\ref{4.5a}).}
\end{align*}
By recycling these arguments, we can easily show that $\liminf_{n\rightarrow\infty}Ef(X_n)\geq Ef(X)-c_1(2+c_2)\epsilon$. Since $\epsilon>0$ was arbitrary but $c_1$ and $c_2$ are fixed, we have the desired conclusion that $Ef(X_n) \rightarrow f(X)$. Since $f \in D_1(\mathbb{B})$ was arbitrary, we now have that (\ref{new.convergence.i}) $\Rightarrow$ (\ref{new.convergence.ii}).

Now we assume (\ref{new.convergence.ii}), and note that
$$W_1(X_n, X) = \sup_{f \in L_1(\mathbb{B})} \Big( E f(X_n) - Ef(X) \Big),$$
where $L_1(\mathbb{B})$ consists of all Lipschitz continuous functions from $\mathbb{B}$ to $\mathbb{R}$ with Lipschitz constant 1. Note that for any $f \in L_1(\mathbb{B})$, $$Ef(X_n) - Ef(X) = E\Big(f(X_n)-f(0)\Big) - E\Big(f(X)-f(0)\Big),$$
and so taking the supremum over the subset $L_1^0(\mathbb{B}) \subset L_1(\mathbb{B})$, where $L_1^0(\mathbb{B})$ are functions in $L_1(\mathbb{B})$ for which $f(0) = 0$, does not change the norms.
By the Lipschitz property, all functions $f \in L_1^0(\mathbb{B})$ also satisfy $|f(X)| \leq \|X\|$.
Hence, since (\ref{new.convergence.ii}) holds, $\limsup_{n \rightarrow \infty}{E\|X_n\| } = E\|X\| <\infty$ since $\|\cdot\|\in D_1(\mathbb{B})$.

Now fix $\epsilon>0$. Let $0<k<\infty$ satisfy $E\big[\|X\| h_k(X)\big]<\epsilon$. Then
$$\limsup_{n\rightarrow\infty}\sup_{f\in L_1^0(\mathbb{B})}E\big[|f(X_n)| h_k(X) + |f(X)| h_k(X) \big]<2\epsilon.$$ 
Hence if we now replace $X_n$ with its projection onto $\mathbb{B}_k=\{x\in\mathbb{B}:\|x\|\leq k+1 \}$, which we denote $\tilde{X}_n$, and also replace $X$ with its projection onto $\mathbb{B}_k$, denoted by $\tilde{X}$, we have the following:
$$W_1(X_n,X)\leq \sup_{f\in L_1^0(\mathbb{B}_k)} Ef(\tilde{X}_n)-Ef(\tilde{X})+2\epsilon.$$
Since $X$ is separable, there exists a compact $A\subset\mathbb{B}$ such that $P(X\not\in A)<\epsilon/k$, where we assume without loss of generality that $0\in A$. Since (\ref{new.convergence.ii}) also implies weak convergence of $X_n$ to $X$, we also have that for every $\delta>0$, 
$\limsup_{n\rightarrow\infty}P\left(X_n\not\in A^\delta\right)<\epsilon/k$, where $A^\delta$ is the open $\delta$-enlargement of $A$, and where we select $\delta=\epsilon/k$. Now if we now let $\tilde{X}_n^\ast$ and $\tilde{X}^\ast$ be the respective projections of $\tilde{X}_n$ and $\tilde{X}$ onto $\overline{A^\delta}$---where for a set $B$, $\overline{B}$ is the closure of $B$---followed by projecting each onto $A$, we can recycle previous arguments in Theorem~\ref{theorem_projection} to verify that for any $f\in L_1^0(\mathbb{B}_k)$, 
$$E|f(\tilde{X}_n^\ast)-f(\tilde{X}_n)| \leq 2\epsilon \text{ and } E|f(\tilde{X}^\ast)-f(\tilde{X})|\leq 2\epsilon.$$
Thus
$$\limsup_{n\rightarrow\infty}W_1(X_n,X)\leq \limsup_{n\rightarrow\infty}\sup_{f\in L_1^0(\mathbb{B}_k\cap A)} Ef(\tilde{X}_n^\ast)-Ef(\tilde{X}^\ast) + 6\epsilon.$$

Since $A$ is compact, it is now easy to verify that $L_1^0(\mathbb{B}_k\cap A)$ is a compact set. Moreover since projections are continuous and bounded, we also have that $Ef(\tilde{X}_n^\ast)-Ef(\tilde{X}^\ast)\rightarrow 0$, as $n\rightarrow\infty$, for all $f\in L_1^0(\mathbb{B}_k\cap A)$. This, combined with the compactness, yields that $\limsup_{n\rightarrow\infty}W_1(X_n,X)\leq 6\epsilon$; and
thus (\ref{new.convergence.i}) follows since $\epsilon$ was arbitrary.

The first three conclusions following have been proved during the course of the equivalence proof above. For the final conclusion, fix $\epsilon,\delta>0$, and let $A\in\mathbb{B}$ be compact and satisfy $P(X\in A)\geq\epsilon$. Now note that for any $\lambda_n\in\Gamma(X_n,X)$, we have that
\begin{eqnarray*}
    P(X_n\in A^\delta)&=&E_{\lambda_n}1\{X_n\in A^\delta\}\\
    &\geq&E_{\lambda_n}1\{X\in A,\|X_n-X\|\leq\delta\}\\
    &\geq& P(X\in A)-E_{\lambda_n}1\{\|X_n-X\|>\delta\}\\
    &\geq& 1-\epsilon-\frac{E_{\lambda_n}\|X_n-X\|}{\delta},\\
\end{eqnarray*}
and now the desired conclusion follows by first minimizing the expectation over $\lambda_n$ and then taking the limit as $n\rightarrow\infty$.
\end{proof}

\subsection{Alternative Wasserstein Distance}\label{s: theory-wass}
This section introduces an alternative Wasserstein distance for Euclidean random variables, the max-sliced Wasserstein distance \cite{deshpande2019max}, which is employed in our simulation study (Section~\ref{s: simulation}) due to its computational efficiency compared to the standard Wasserstein distance.

For any two random variables $X,Y \in \mathbb{R}^d$, let $\Gamma(X,Y) $ be the set of all hypothetical distributions of $\lambda$ of $X$ and $Y$ so that the marginal of $X$ in $\lambda$ matches the true distribution of $X$ and the marginal of $Y$ in $\lambda$ matches the true distribution of $Y$. 
For $1 \leq p < \infty, $ the Wasserstein metric $W_p$ quantifies the distance between two probability distributions in a $d$-dimensional space. It is defined as:
$$W_p(X,Y) = \Big( \inf_{\lambda \in \Gamma(X, Y)} \int \|x-y \|^p \lambda(x,y) \,dx \,dy \Big)^{1/p},$$
where $\| \cdot \|$ is the Euclidean distance. 

Due to the computational challenges associated with the Wasserstein metric, the max-sliced Wasserstein distance has been proposed \cite{deshpande2019max}. This alternative metric offers increased computational efficiency while maintaining desirable properties for many applications \cite{paty2019subspace, bayraktar2021strong}. Let $S_d = \{ t \in \mathbb{R}^d: \|t\|=1\}$, the $d$-dimensional sphere. For $1 \leq d < \infty$, the max-sliced Wasserstein distance, denoted by $W_{p}^*(X,Y)$, is defined as: 
$$W_{p}^*(X,Y)=\sup_{t \in S_d} W_p(t'X, t'Y).$$
In the case where $p=1$, Bayraktar and Guo (2021) \cite{bayraktar2021strong} established the strong equivalence of the $1$-Wasserstein distance $W_1(X,Y)$ and $W_{1}^*(X,Y)$ for any dimension $d \geq 1$, which theorem we give here:
\begin{theorem} \label{theorem_maxswass_equiv}
    $W_{1}^*(X,Y)$ and $W_1(X,Y)$ are strongly equivalent for any two random variables $X,Y \in \mathbb{R}^d$ for all $d \geq 1$, i.e. there exists $C_d \geq 1$ such that
    $$W_{1}^*(X,Y) \leq W_{1}(X,Y) \leq C_d W_{1}^*(X,Y). $$
\end{theorem}

Let $D_d=D_1(\mathbb{R}^d)$. Building upon the definition of $D_d$, we present the following theorem:

\begin{theorem}
\label{theorem_maxswass_equiv2}
Let $X,X_n,Y$ be random variables in $\mathbb{R}^d$ for integer $1 \leq d < \infty$. The following are true.
\begin{enumerate}[label=(\roman*)]
\item[(a)]  Let $\{X_n\}$ be a sequence with $E\|X_n\| < \infty$ for all $n \geq 1$. Then the following are equivalent:

\begin{enumerate}[label=(\roman*)]
    \item \label{i_wass} $W_1(X_n,X) \rightarrow 0$.
    \item \label{ii_wass} $W_1^*(X_n,X) \rightarrow 0$.
    \item \label{iii_wass} For every $f\in D_d$, $Ef(X_n) \rightarrow Ef(X)<\infty$.
\end{enumerate}

\item[(b)] for any $\epsilon >0$, let $T_\epsilon \subset S_d$ be a finite subset satisfying $\sup_{s \in S_d} \inf_{t\in T_\epsilon} \| s-t\| \leq \epsilon$. Then 
$$|W_1^*(X,Y) - \sup_{t \in T_\epsilon} W_1(t'X, t'Y) |\leq \epsilon (E\| X \| + E\| Y \|).$$ 
\end{enumerate}
\end{theorem}
Note that (\ref{iii_wass}) implies both $X_n \rightsquigarrow X$ (convergence in distribution) and $E\|X_n\| \rightarrow E\|X\| < \infty$, and much more, as we discussed in the previous section. 

\begin{proof}
\begin{enumerate}[label=(\roman*)]
\item[\emph{Proof of} (a).] The equivalence of (\ref{i_wass}) and (\ref{ii_wass}) follows directly from Theorem~\ref{theorem_maxswass_equiv}, and the equivalence of (\ref{i_wass}) and (\ref{iii_wass}) follows directly from Theorem~\ref{new.convergence}.

\item[\emph{Proof of} (b).] For any $s,t \in S_d$, and any $\lambda \in \Gamma(X,Y)$,
\begin{align*}
    \Big| E_{\lambda}|t'X-t'Y| -  E_{\lambda}|s'X-s'Y|  \Big| & \leq E_{\lambda} \Big| |t'X-t'Y|-|s'X-s'Y|\Big| \\
    & \leq E_\lambda \Big| \|t-s\|\cdot \|X-Y\| \Big|\\
    & = \|t-s\|E_{\lambda}\|X-Y\| \\
    & \leq \|t-s\|\Big( E\|X\| + E\|Y\|  \Big).
\end{align*}
Thus, since
$$E_{\lambda} |t'X-t'Y| = E_{\lambda}|s'X-s'Y| + E_{\lambda}\Big(|t'X-t'Y|-|s'X-s'Y| \Big),$$
we obtain:
$$E_{\lambda} |t'X-t'Y| \leq E_{\lambda}|s'X-s'Y| + \|t-s\|  \Big(E\|X\|+E\|Y\| \Big).$$

By taking infimum over $\lambda \in \Gamma(X,Y)$ on the left side then right, we obtain that
$$W_1(t'X, t'Y) \leq W_1 (s'X, s'Y) + \|t-s\| \cdot \Big( E\|X\| + E\|Y\|   \Big).$$
By reversing the roles of $s$ and $t$, we obtain that
\begin{align*}
    &W_1(s'X, s'Y) \leq W_1 (t'X, t'Y) + \|t-s\| \cdot \Big( E\|X\| + E\|Y\|   \Big). \\
    &\Rightarrow \Big| W_1(s'X,s'Y) - W_1(t'X, t'Y)\Big| \leq \|t-s\| \cdot \Big( E\|X\| + E\|Y\|\Big), \\
    &\Rightarrow \text{ (b) is true, and the proof is complete.}
\end{align*}
\end{enumerate}
\end{proof}

This theorem highlights an important result. It states that by choosing a finite subset $T$ that covers $S_d$ well, we can approximate the true Wasserstein distance $W_1(X,Y)$ using $\sup_{t \in T} W_1(t'X, t'Y)$. This allows for a more practical and computationally tractable method compared to directly calculating the Wasserstein distance. 

\section{Simulation}
\label{s: simulation}

In this section, we present the performance of our proposed algorithm. We conceptualized a 2-dimensional state space and a 2-dimensional reward space. The transitions and rewards are generated based on the following rules:

\begin{align*}
    &S_1, S_2  \in \{1,2,3, \dots, 15\}, \\
    &R_1, R_2  \in [-15, 15], \\
    {}\\
    &S_1' |s_1, a  \sim \left\{
  \begin{array}{@{}ll@{}}
    Chisq(s_1)        & w.p. \ 0.75  \text{ if }  a=-1,   \ 0.25  \text{ if }  a=1, \\
    N(0.1s_1 + 8 , 1) & w.p. \ 0.25  \text{ if }  a=-1,   \ 0.75  \text{ if }  a=1,
  \end{array}\right. \\
    &S_1'  = \big[ round(S_1') \big]_{1} ^{15},  \\
    {}\\
    &R_1(s_1, a, s_1') \sim N\big(s_1' - s_1 -0.2 \cdot I(a=1) , 1 \big), \\
    &R_1 = \big[ R_1 \big]_{-15} ^{15},  \\
    {}\\
    &S_2' |s_1, s_2, a \sim \left\{ 
        \begin{array}{@{}ll@{}}
             Exp(\theta = \lfloor 0.25s_1 + s_2\rfloor_1) & w.p. \ 0.75  \text{ if }  a=-1,   \ 0.25  \text{ if }  a=1, \\
             Uniform(1, 10) & w.p. \ 0.25  \text{ if }  a=-1,   \ 0.75  \text{ if }  a=1, 
        \end{array}\right. \\
    &S_2'  = \big[ round(X_2') \big]_{1} ^{15},  \\
    {}\\
    &R_2(s_2, a, s_2') \sim N\big(s_2' - s_2, 1 \big), \\
    &R_2 = \big[ R_2  \big]_{-15} ^{15}.  \\    
\end{align*}

We considered deterministic stationary policies. These policies are characterized by three parameters: $\beta_0$, $\beta_1$, and $sgn\in\{-1,1\}$ given two state variables $s1$ and $s2$. $\beta_0$ and $\beta_1$ represent linear translation coefficients. $sgn$ controls the sign of the translation. 
$$
 \pi(s_1, s_2) =\left\{
  \begin{array}{@{}ll@{}}
    \ \ \  1, & \text{if}\ sgn(\beta_0 + \beta_1 s_1 + s_2) \geq 0 \\
    -1, & \text{if}\ sgn(\beta_0 + \beta_1 s_1 + s_2) < 0
  \end{array}\right.
$$

Note that the transition probability $P(\cdot | s, a)$ and reward function $R(s, a=\pi(s), s')$ are integral parts of Algorithm~\ref{algo1}. Section~\ref{s: simulation-1} presents the algorithm's performance under the assumption of known transition probability and reward function. In contrast, Section~\ref{s: simulation-2} demonstrates the algorithm's performance when the true distributions are unknown. Finally, in Section~\ref{s: simulation-3}, we showcase the algorithm's performance in searching for the optimal policy based on a specified utility function.

To quantify the similarity between the two $2$-dimensional return distributions, we utilized the max-sliced Wasserstein distance with finite support introduced in Section~\ref{s: theory-wass}:

\begin{align}
    W_\Theta (Z_1, Z_2) = \sup_{\theta \in \Theta}W_\theta(Z_1, Z_2) 
                         =\sup_{\theta \in \Theta}W_1\Big((\cos{\theta} \ \sin{\theta}) Z_1, (\cos{\theta} \ \sin{\theta}) Z_2\Big), \label{5.1}
\end{align}
where $W_\theta(Z_1, Z_2) = W_1\Big((\cos{\theta} \ \sin{\theta}) Z_1, (\cos{\theta} \ \sin{\theta}) Z_2\Big)$ represents the Wasserstein distance between two distributions projected onto a line with a slope of $\theta$ passing through the origin. We construct $\Theta$ as a set of 60 equally spaced angles within the range $[0, \pi)$. 

As the true distribution of random return is not directly derived from the conceptualized stochastic process, we obtained the empirical distribution by generating $10^4$ returns given initial states and a specified policy. These empirical distributions are compared with the corresponding estimated distribution.

In Algorithm~\ref{algo1}, we set the number of draws for random returns, denoted as $n_{sample}$, to be $10^3$. We selected the center of the rectangles to be represented by ${z_\alpha} = \{(z_1,z_2) : z_1 \in \mathbb{Z}_1$ and $z_2 \in \mathbb{Z}_2\}$, where $\mathbb{Z}_1 = \mathbb{Z}_2 = \{ x:x=-25+i \cdot 50/40 \text{ for } i = 0,1,2,..., 40\}$. Consequently, the total number of categories $N$ is 41 $\times$ 41. 

\subsection{Scenario 1} \label{s: simulation-1}
$P(\cdot | s, a)$ and $R(s, a=\pi(s), s')$ are known.

In this section, we examine 4 policies denoted as $\pi(\beta_0, \beta_1, s)$. For Policy 1 and Policy 2, we consider $\pi(-7.5, 0.5, -1)$ and $\pi(-7.5, 0.5, +1)$, respectively. Policy 3, $\pi(15, 2, -1)$, and Policy 4, $\pi(15, 2, +1)$, represent policies orthogonal to the first two. The initial distributions were generated using a 2D uniform distribution in the ranges (-12.5, 12.5). For comparison, we generated empirical distributions with an initial state (1,1) for each policy and observed the finite-support max-sliced Wasserstein distance $W_\Theta$ (\ref{5.1}) by incrementally increasing $n_{repeat}$. The distance paths corresponding to these policies are illustrated in Figure~\ref{fig:distance_known}. $W_\Theta$ values stay below 0.5 after 10 iteration steps, and the distances are 0.186, 0.179, 0.121, and 0.206 after 20 iterations. The results show that the $W_\Theta$ values for all policies remain below 0.5 after 10 iterations. After 20 iterations, the distances converge further, reaching specific values of 0.186, 0.179, 0.121, and 0.206 for Policies 1, 2, 3, and 4, respectively.

Additionally, for each policy, we visually present the evolution of the estimated distribution over iterations $n_{repeat}$ from Figure~\ref{fig:sc1_policy1_known_dist} to Figure~\ref{fig:sc1_policy4_known_dist}. These figures provide valuable insights into how the estimated distribution changes as the number of iterations increases. The simulation results affirm that Algorithm~\ref{algo1} accurately estimates the distribution of random returns when the transition kernel $P(\cdot | s, a)$ and reward function $R(s, a=\pi(s), s')$ are known.

\begin{figure}
    \centering
    \includegraphics[scale=0.5]{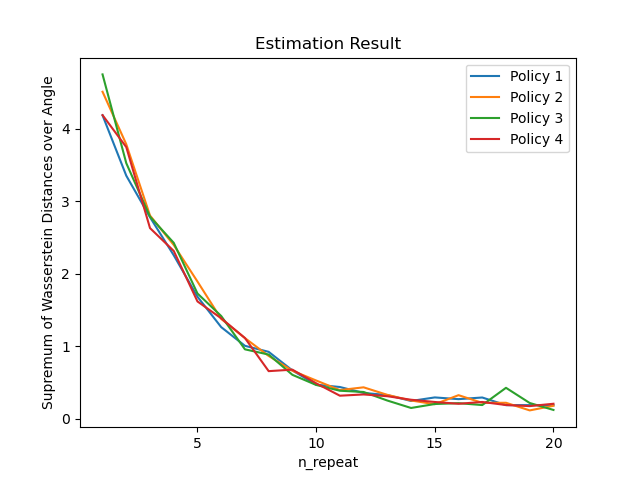}
    \caption{Distance Path in Scenario 1.}
    \label{fig:distance_known}
\end{figure}

\begin{figure}
     \centering
     \begin{subfigure}[b]{0.24\textwidth}
         \centering
         \includegraphics[width=\textwidth]{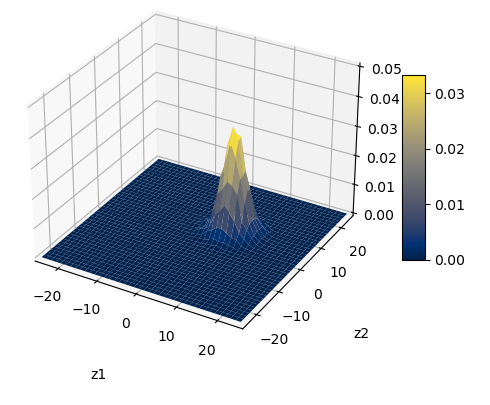}
         \caption{Empirical}
         \label{fig:sc1_policy1_emp}
     \end{subfigure}
     \hfill
     \begin{subfigure}[b]{0.24\textwidth}
         \centering
         \includegraphics[width=\textwidth]{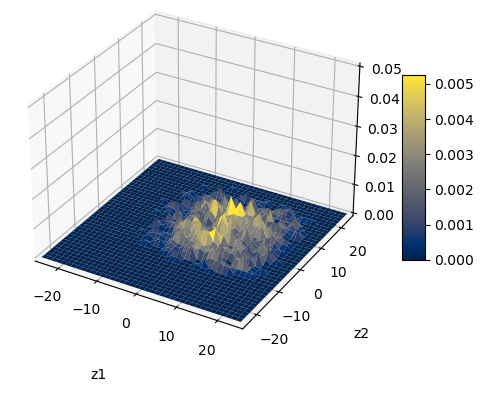}
         \caption{Iteration 1}
         \label{fig:sc1_policy1_rep1}
     \end{subfigure}
     \hfill
     \begin{subfigure}[b]{0.24\textwidth}
         \centering
         \includegraphics[width=\textwidth]{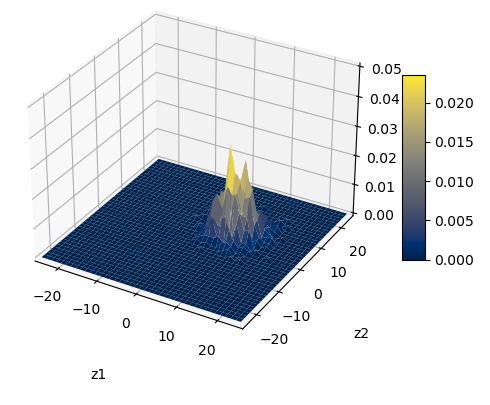}
         \caption{Iteration 8}
         \label{fig:sc1_policy1_rep8}
     \end{subfigure}
     \hfill
     \begin{subfigure}[b]{0.24\textwidth}
         \centering
         \includegraphics[width=\textwidth]{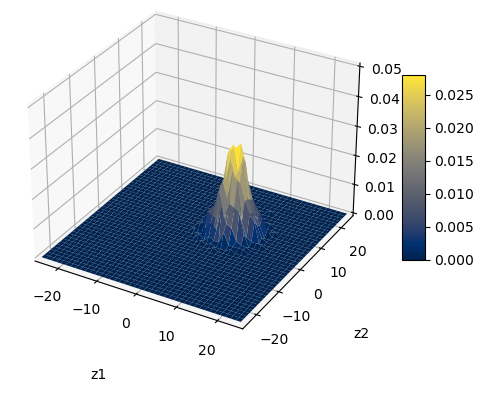}
         \caption{Iteration 15}
         \label{fig:sc1_policy1_rep15}
     \end{subfigure}
        \caption{Policy 1. Empirical vs. Estimated Distribution by Algorithm 1 in Scenario 1}
        \label{fig:sc1_policy1_known_dist}
\end{figure}

\begin{figure}
     \centering
     \begin{subfigure}[b]{0.24\textwidth}
         \centering
         \includegraphics[width=\textwidth]{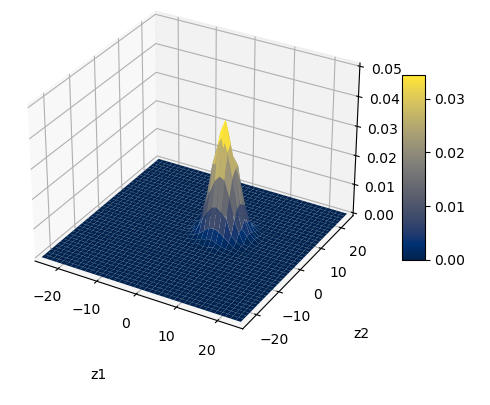}
         \caption{Empirical}
         \label{fig:sc1_policy2_emp}
     \end{subfigure}
     \hfill
     \begin{subfigure}[b]{0.24\textwidth}
         \centering
         \includegraphics[width=\textwidth]{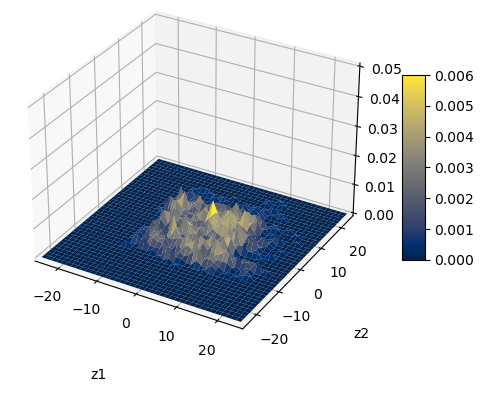}
         \caption{Iteration 1}
         \label{fig:sc1_policy2_rep1}
     \end{subfigure}
     \hfill
     \begin{subfigure}[b]{0.24\textwidth}
         \centering
         \includegraphics[width=\textwidth]{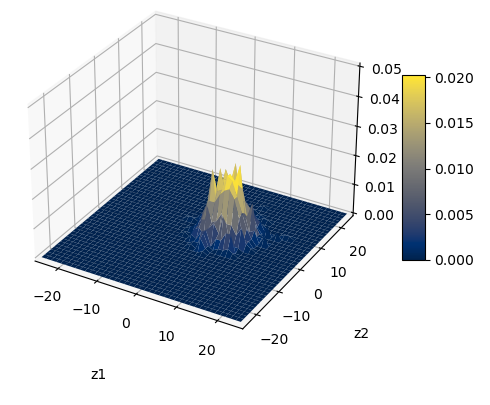}
         \caption{Iteration 8}
         \label{fig:sc1_policy2_rep8}
     \end{subfigure}
     \hfill
     \begin{subfigure}[b]{0.24\textwidth}
         \centering
         \includegraphics[width=\textwidth]{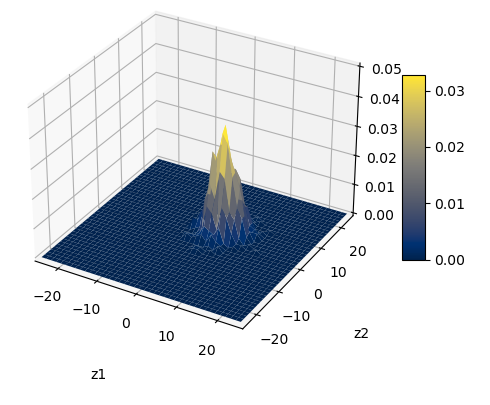}
         \caption{Iteration 15}
         \label{fig:sc1_policy2_rep15}
     \end{subfigure}
        \caption{Policy 2. Empirical vs. Estimated Distribution by Algorithm 1 in Scenario 1}
        \label{fig:sc1_policy2_known_dist}
\end{figure}

\begin{figure}
     \centering
     \begin{subfigure}[b]{0.24\textwidth}
         \centering
         \includegraphics[width=\textwidth]{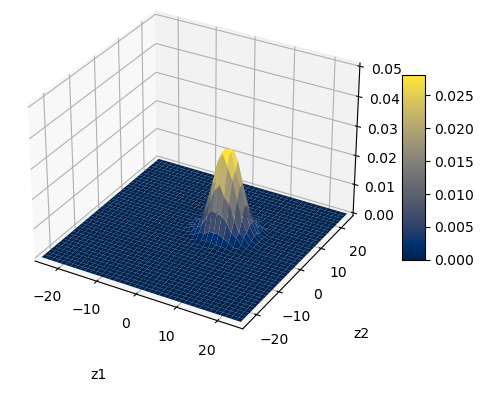}
         \caption{Empirical}
         \label{fig:sc1_policy3_emp}
     \end{subfigure}
     \hfill
     \begin{subfigure}[b]{0.24\textwidth}
         \centering
         \includegraphics[width=\textwidth]{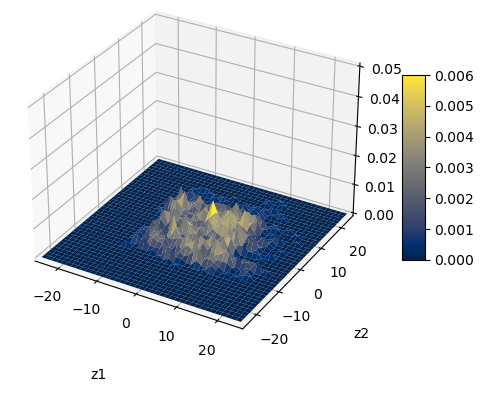}
         \caption{Iteration 1}
         \label{fig:sc1_policy3_rep1}
     \end{subfigure}
     \hfill
     \begin{subfigure}[b]{0.24\textwidth}
         \centering
         \includegraphics[width=\textwidth]{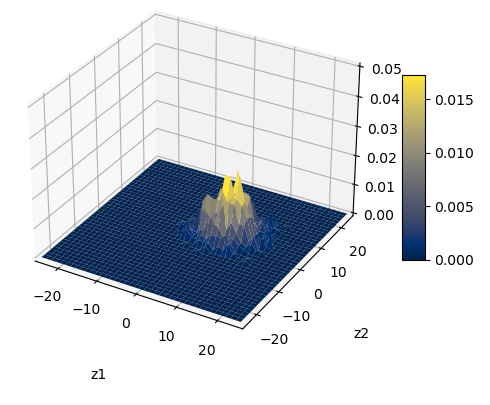}
         \caption{Iteration 8}
         \label{fig:sc1_policy3_rep8}
     \end{subfigure}
     \hfill
     \begin{subfigure}[b]{0.24\textwidth}
         \centering
         \includegraphics[width=\textwidth]{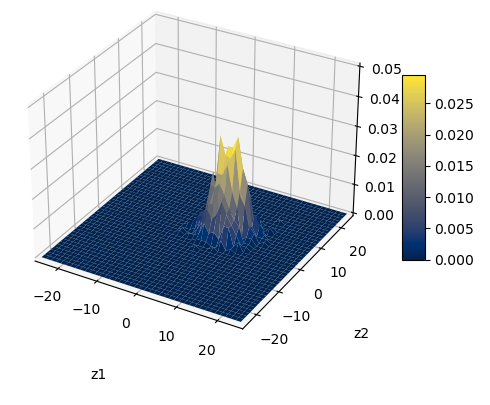}
         \caption{Iteration 15}
         \label{fig:sc1_policy3_rep15}
     \end{subfigure}
        \caption{Policy 3. Empirical vs. Estimated Distribution by Algorithm 1 in Scenario 1}
        \label{fig:sc1_policy3_known_dist}
\end{figure}

\begin{figure}
     \centering
     \begin{subfigure}[b]{0.24\textwidth}
         \centering
         \includegraphics[width=\textwidth]{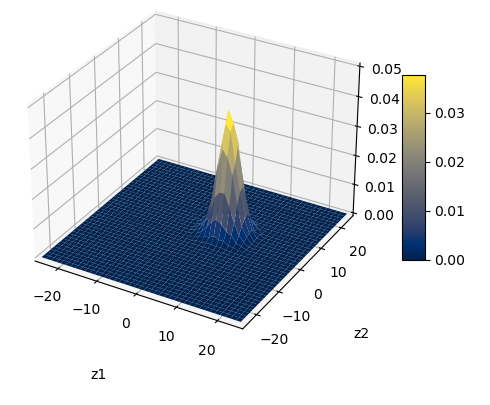}
         \caption{Empirical}
         \label{fig:sc1_policy4_emp}
     \end{subfigure}
     \hfill
     \begin{subfigure}[b]{0.24\textwidth}
         \centering
         \includegraphics[width=\textwidth]{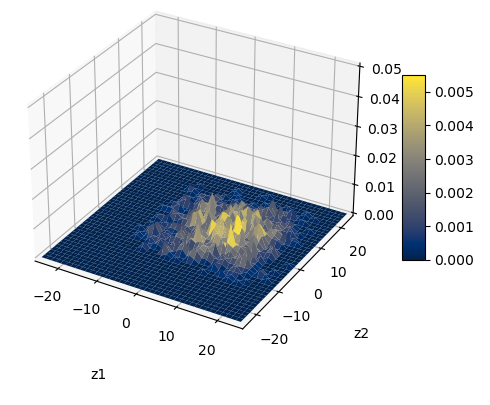}
         \caption{Iteration 1}
         \label{fig:sc1_policy4_rep1}
     \end{subfigure}
     \hfill
     \begin{subfigure}[b]{0.24\textwidth}
         \centering
         \includegraphics[width=\textwidth]{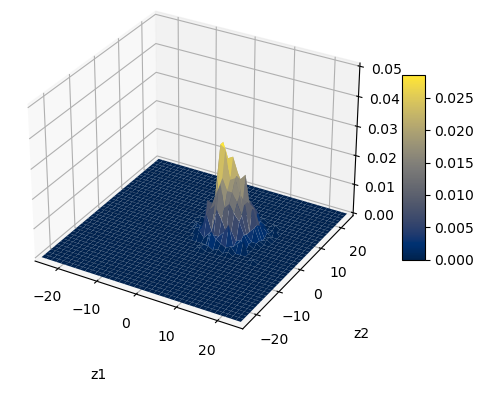}
         \caption{Iteration 8}
         \label{fig:sc1_policy4_rep8}
     \end{subfigure}
     \hfill
     \begin{subfigure}[b]{0.24\textwidth}
         \centering
         \includegraphics[width=\textwidth]{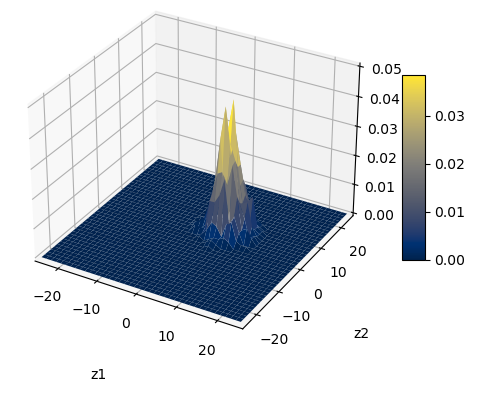}
         \caption{Iteration 15}
         \label{fig:sc1_policy4_rep15}
     \end{subfigure}
        \caption{Policy 4. Empirical vs. Estimated Distribution by Algorithm 1 in Scenario 1}
        \label{fig:sc1_policy4_known_dist}
\end{figure}

\subsection{Scenario 2} \label{s: simulation-2} 
$P(\cdot | s, a)$ and $R(s, a=\pi(s), s')$ are unknown. 

In reality, the transition kernel $P(\cdot | s, a)$ is often unknown, and the reward function $R(s, a=\pi(s), s')$ may not be known in certain settings. In this section, we present simulation results in such scenarios where these elements are unknown. To address this, we estimate both the transition kernel and reward function from observed data. We generated $n_{trajectory}$ transition trajectories, each with a length of $100$ steps. By incrementally increasing $n_{trajectory}$, we updated the transition probability and reward function, incorporating information from every additional set of $100$ trajectories.

We employed separate neural network models to estimate the state transition probabilities, $P(s_1' | s_1, s_2, a)$ and $P(s_2' | s_1, s_2, a)$, where $(s_1, s_2)$ represents the current state, $a$ denotes the action taken, and $(s_1', s_2')$ represents the next state. Both models shared identical architecture and training parameters to ensure consistency in their learning process.

The neural network architecture consisted of three layers: an input layer, two hidden layers, and an output layer. The input layer received a $3$-dimensional vector containing the current state $(s_1, s_2)$ and the action $(a)$. The first hidden layer comprised 64 nodes with rectified linear unit (ReLU) activation functions. The second hidden layer contained 16 nodes with ReLU activations, to which the action input $(a)$ was concatenated, resulting in a total of 17 nodes in this layer. Finally, the output layer consisted of 15 nodes, each employing a softmax activation function. The output layer's nodes corresponded to the probabilities of transitioning to each possible next state value $\{1, 2, ..., 15\}$.

The models were trained using the cross-entropy loss function and the Adam optimizer with a learning rate of $5e$-$6$. The training data was presented in batches of 32 samples, and training proceeded for a maximum of 100 epochs. To prevent overfitting, the data was split into a training set (90\%) and a validation set (10\%). Early stopping was implemented, terminating training if the validation loss did not improve for 5 consecutive epochs.

Multivariate linear regression was employed to estimate the reward function given transition information $R(r_1, r_2 | s_1, s_2, a, s_1', s_2')$, where $(r_1, r_2)$ are rewards. The model treated $(s_1, s_2, a, s_1', s_2')$ as the independent variables and aimed to predict the corresponding rewards $(r_1, r_2)$.

Given the estimated transition probability and reward function, we employed Algorithm~\ref{algo1} to estimate the return distribution by policy, setting $n_{repeat} = 20$. For each policy, we illustrate the distance path $W_\Theta$ between the estimated return distribution and empirical distribution in Figure~\ref{fig:distance_unknown}. The evolution of the estimated distributions over the number of observed trajectories $n_{trajectory}$ is depicted in Figure~\ref{fig:sc2_policy1_unknown_dist} to \ref{fig:sc2_policy4_unknown_dist}. 

These results demonstrate the effectiveness of the proposed algorithm in estimating the return distribution for different policies even when the true transition probabilities and reward functions are unknown. This highlights the algorithm's ability to perform offline policy evaluation using only observed data.

\begin{figure}
    \centering
    \includegraphics[scale=0.5]{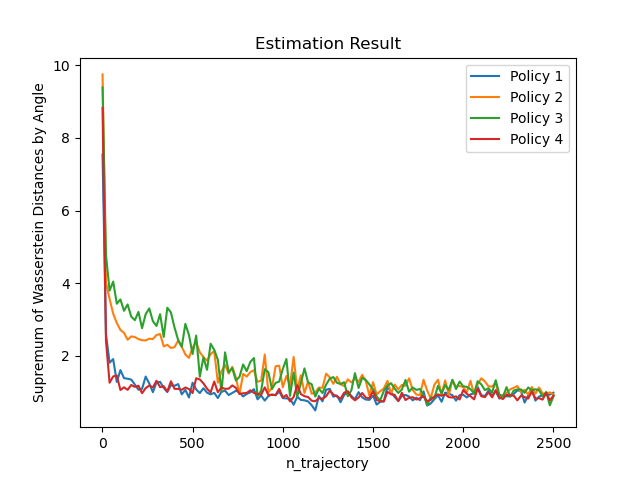}
    \caption{Distance Path in Scenario 2.}
    \label{fig:distance_unknown}
\end{figure}

\begin{figure}
     \centering
     \begin{subfigure}[b]{0.24\textwidth}
         \centering
         \includegraphics[width=\textwidth]{figure/simulation1/1_empirical.png}
         \caption{Empirical}
         \label{fig:sc2_policy1_emp}
     \end{subfigure}
     \hfill
     \begin{subfigure}[b]{0.24\textwidth}
         \centering
         \includegraphics[width=\textwidth]{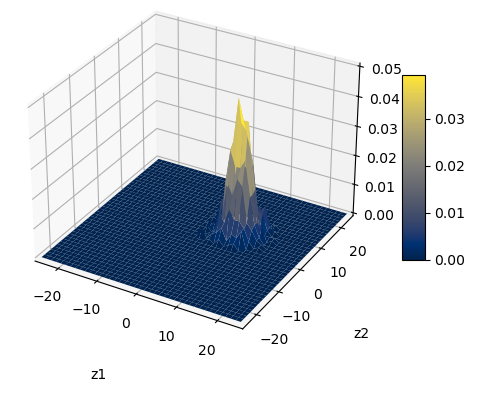}
         \caption{$n_{trajectory} = 100$}
         \label{fig:sc2_policy1_traj100}
     \end{subfigure}
     \hfill
     \begin{subfigure}[b]{0.24\textwidth}
         \centering
         \includegraphics[width=\textwidth]{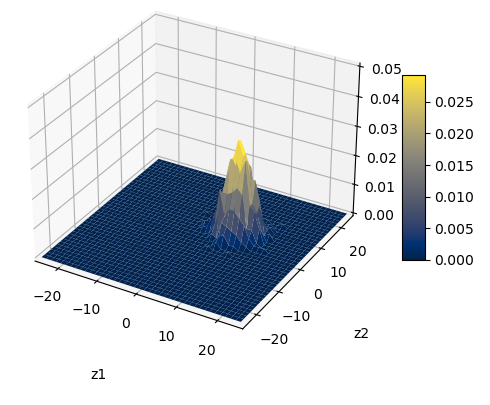}
         \caption{$n_{trajectory} = 500$}
         \label{fig:sc2_policy1_traj500}
     \end{subfigure}
     \hfill
     \begin{subfigure}[b]{0.24\textwidth}
         \centering
         \includegraphics[width=\textwidth]{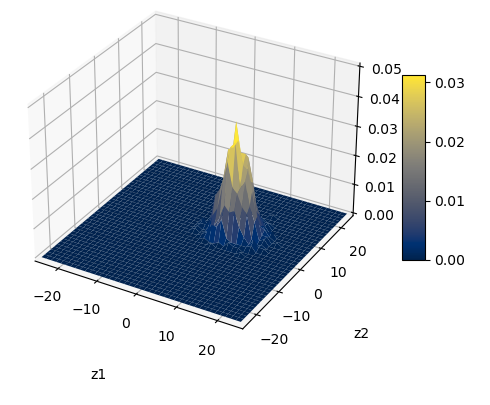}
         \caption{$n_{trajectory} = 1000$}
         \label{fig:sc2_policy1_traj1000}
     \end{subfigure}
        \caption{Policy 1. Empirical vs. Estimated Distribution by Algorithm 1 in Scenario 2}
        \label{fig:sc2_policy1_unknown_dist}
\end{figure}

\begin{figure}
     \centering
     \begin{subfigure}[b]{0.24\textwidth}
         \centering
         \includegraphics[width=\textwidth]{figure/simulation1/2_empirical.png}
         \caption{Empirical}
         \label{fig:sc2_policy2_emp}
     \end{subfigure}
     \hfill
     \begin{subfigure}[b]{0.24\textwidth}
         \centering
         \includegraphics[width=\textwidth]{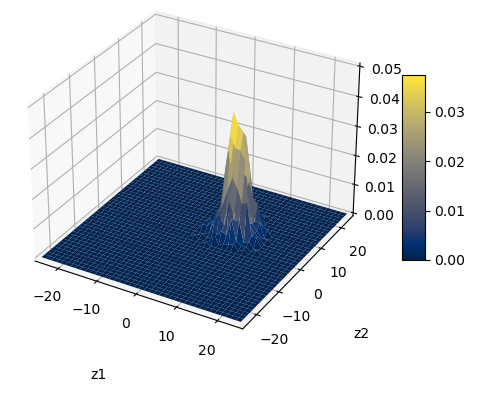}
         \caption{$n_{trajectory} = 100$}
         \label{fig:sc2_policy2_traj100}
     \end{subfigure}
     \hfill
     \begin{subfigure}[b]{0.24\textwidth}
         \centering
         \includegraphics[width=\textwidth]{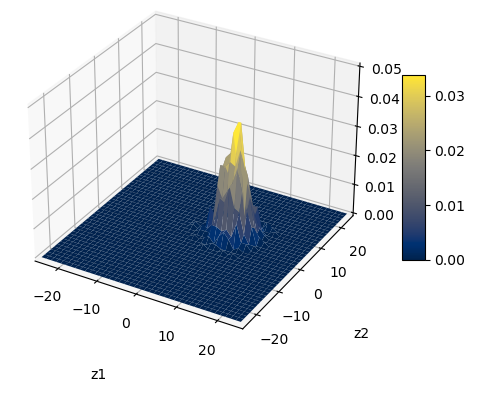}
         \caption{$n_{trajectory} = 500$}
         \label{fig:sc2_policy2_traj500}
     \end{subfigure}
     \hfill
     \begin{subfigure}[b]{0.24\textwidth}
         \centering
         \includegraphics[width=\textwidth]{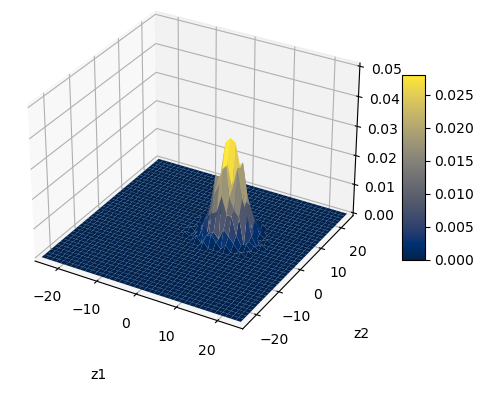}
         \caption{$n_{trajectory} = 1000$}
         \label{fig:sc2_policy2_traj1000}
     \end{subfigure}
        \caption{Policy 2. Empirical vs. Estimated Distribution by Algorithm 1 in Scenario 2}
        \label{fig:sc2_policy2_unknown_dist}
\end{figure}

\begin{figure}
     \centering
     \begin{subfigure}[b]{0.24\textwidth}
         \centering
         \includegraphics[width=\textwidth]{figure/simulation1/3_empirical.png}
         \caption{Empirical}
         \label{fig:sc2_policy3_emp}
     \end{subfigure}
     \hfill
     \begin{subfigure}[b]{0.24\textwidth}
         \centering
         \includegraphics[width=\textwidth]{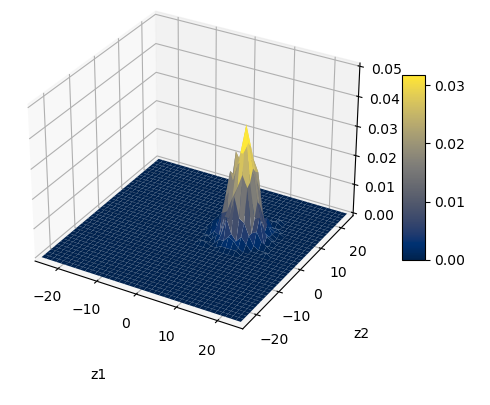}
         \caption{$n_{trajectory} = 100$}
         \label{fig:sc2_policy3_traj100}
     \end{subfigure}
     \hfill
     \begin{subfigure}[b]{0.24\textwidth}
         \centering
         \includegraphics[width=\textwidth]{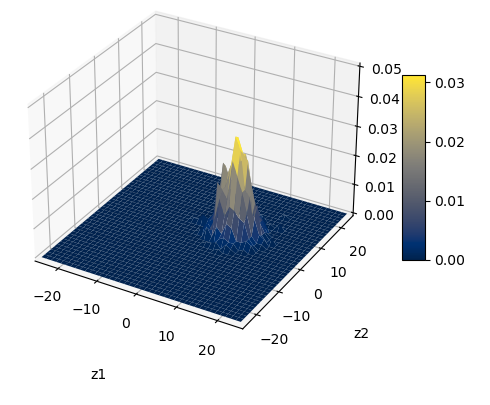}
         \caption{$n_{trajectory} = 500$}
         \label{fig:sc2_policy3_traj500}
     \end{subfigure}
     \hfill
     \begin{subfigure}[b]{0.24\textwidth}
         \centering
         \includegraphics[width=\textwidth]{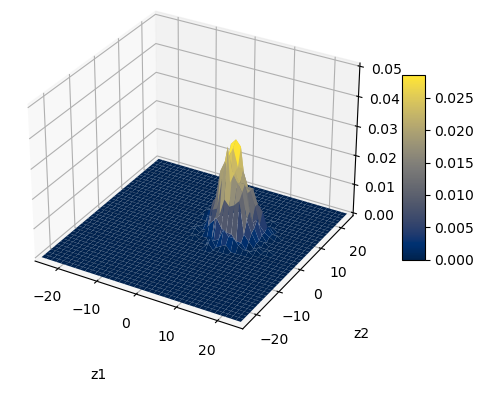}
         \caption{$n_{trajectory} = 1000$}
         \label{fig:sc2_policy3_traj1000}
     \end{subfigure}
        \caption{Policy 3. Empirical vs. Estimated Distribution by Algorithm 1 in Scenario 2}
        \label{fig:sc2_policy3_unknown_dist}
\end{figure}

\begin{figure}
     \centering
     \begin{subfigure}[b]{0.24\textwidth}
         \centering
         \includegraphics[width=\textwidth]{figure/simulation1/4_empirical.png}
         \caption{Empirical}
         \label{fig:sc2_policy4_emp}
     \end{subfigure}
     \hfill
     \begin{subfigure}[b]{0.24\textwidth}
         \centering
         \includegraphics[width=\textwidth]{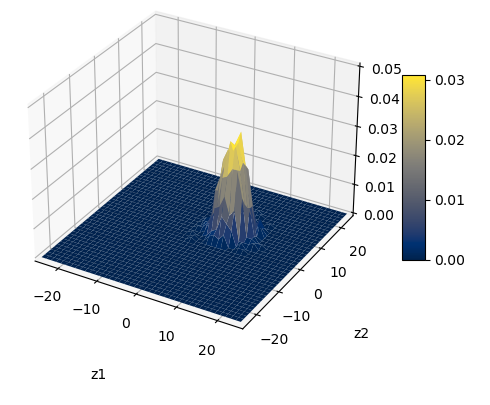}
         \caption{$n_{trajectory} = 100$}
         \label{fig:sc2_policy4_traj100}
     \end{subfigure}
     \hfill
     \begin{subfigure}[b]{0.24\textwidth}
         \centering
         \includegraphics[width=\textwidth]{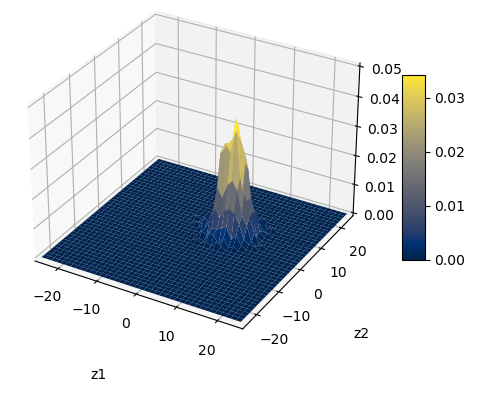}
         \caption{$n_{trajectory} = 500$}
         \label{fig:sc2_policy4_traj500}
     \end{subfigure}
     \hfill
     \begin{subfigure}[b]{0.24\textwidth}
         \centering
         \includegraphics[width=\textwidth]{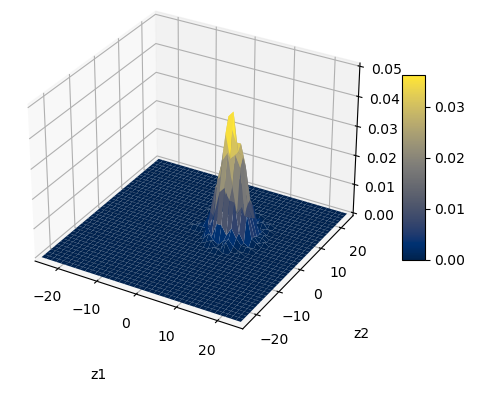}
         \caption{$n_{trajectory} = 1000$}
         \label{fig:sc2_policy4_traj1000}
     \end{subfigure}
        \caption{Policy 4. Empirical vs. Estimated Distribution by Algorithm 1 in Scenario 2}
        \label{fig:sc2_policy4_unknown_dist}
\end{figure}

\subsection{Scenario 3} \label{s: simulation-3} Optimal policy search with unknown transition dynamics.

This section investigates the performance of our proposed algorithm in identifying the optimal policy when both the transition probability, $P(s' | s, a)$, and the reward function, $R(s, a = \pi(s), s')$, are unknown. We leverage a utility function, $\phi(Z) = median(Z_1) + 20P(Z_2 > 5)$, to evaluate the policy. The utility function maximizes the median value of the first element $Z_1$ in the return distribution of $Z$ while also placing a significant weight (20x) on the probability of the second element $Z_2$ exceeding a threshold of 5.

We considered a set of 200 candidate linear policies denoted by the set $\Pi$. To ensure comprehensive exploration of the state space, we first randomly selected 100 distinct pairs of coefficients $(\beta_0, \beta_1)$. These pairs were chosen such that they uniquely partition the state space, leading to a diverse set of policy behaviors. For each unique pair of coefficients, we evaluated both sign options ($sgn = -1$ and $sgn = 1$), resulting in a total of 200 candidate policies.
For every 100 newly observed trajectories (each with a length of 100), we updated the estimated transition probability and reward function based on the newly observed data. Given these, we updated the return distributions $Z^\pi$ for all policies in $\Pi$.

We identified the policy $\hat{\pi}^*$ exhibiting the highest estimated utility $\phi(Z^\pi)$ at each update step. To assess the selected policy $\hat{\pi}^*$ under the true (but unknown) dynamics, we performed a simulation-based evaluation. We simulated the behavior of $\hat{\pi}^*$ for a total of $10^4$ trajectories, each consisting of 100 steps, starting from the initial state $s=(1,1)$. To generate these trajectories, we leveraged the true (but unknown) transition probability and reward function.

We calculated the utility function $\phi(Z^{\hat{\pi}^*}(s))$ based on the simulated data and plotted the path of this value across update steps in Figure~\ref{fig:utility_path}. Additionally, the figure includes horizontal lines representing the percentiles of the utilities obtained by all policies in the set $\Pi$. These percentiles are calculated by evaluating $\phi(Z^{\pi}(s))$ for each $\pi \in \Pi$ using the same method as for $\phi(Z^{\hat{\pi}^*}(s))$. As evident from Figure~\ref{fig:utility_path}, the estimated optimal policy, $\hat{\pi}^*$, achieves a utility level approximately at the 95th percentile. This demonstrates the effectiveness of our proposed distributional RL-based algorithm in tackling challenges that are not readily addressed by conventional RL methods.

\begin{figure}
    \centering
    \includegraphics[scale=0.5]{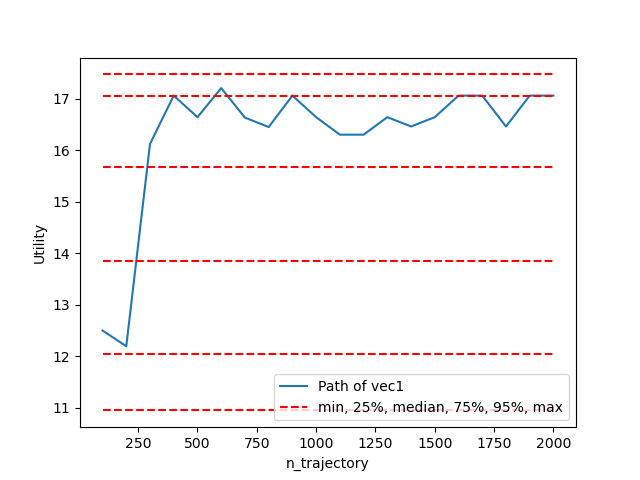}
    \caption{Utility Path with percentiles and min/max in Scenario 3.}
    \label{fig:utility_path}
\end{figure}



\section{Discussion}
\label{s: discussion}

This paper establishes robust and generalizable theoretical foundations for distributional reinforcement learning. We demonstrate the contraction property of the distributional Bellman operator even when the reward space is an infinite-dimensional Banach space.  Furthermore, we show that high- or infinite-dimensional reward behavior can be effectively approximated using a lower-dimensional Euclidean space. Leveraging these theoretical insights, we propose a novel distributional RL algorithm capable of tackling problems beyond the reach of conventional reinforcement learning approaches.

Recent advancements in deep learning, particularly the success of transformer-based models in capturing sequential dependencies inherent in RL tasks \cite{vaswani2017attention, parisotto2020stabilizing, chen2021decision}, offer promising avenues for improvement on our proposed algorithm. Building upon these advancements, a key direction for future research involves replacing the table-based updates in our algorithm with a transformer-based approach. We envision a transformer architecture that directly estimates the distribution of the random return, given the sequence of states, policy parameters, and relevant return distributions as inputs. This approach has the potential to simplify the estimation process and lead to significant performance gains.

\bibliographystyle{unsrt}  
\bibliography{main}

\end{document}